%% file: main.tex
\documentclass[a4paper,12pt]{extarticle}
\usepackage{geometry}
\usepackage{graphicx}
\usepackage{tikz}
\usepackage{color}
\usepackage{booktabs}
\usepackage[font={small}]{caption}
\usepackage{braket}
\usepackage{multirow}
\usepackage{xcolor}
\usepackage{xspace}
\usepackage{subfig}
\usepackage{wrapfig}
\usepackage[version=3]{mhchem}
\usepackage{amsmath}
\usepackage{amsthm}
\usepackage{amssymb,mathtools}
\usepackage{footmisc}

\theoremstyle{plain}

\newtheorem{corollary}{Corollary}
\newtheorem{proposition}{Proposition}

\theoremstyle{definition}

\theoremstyle{remark}
\newtheorem{remark}{Remark}

\usepackage[pagebackref,breaklinks,colorlinks]{hyperref}

\DeclareMathOperator{\vect} {vec}
\DeclareMathOperator{\trace} {trace}
\DeclareMathOperator{\diag} {diag}
\newcommand{\vone} {\mathbf{1}}
\newcommand{\tr} {{\mathsf{T}}}

\usepackage[capitalize]{cleveref}
\crefname{section}{Sec.}{Secs.}
\Crefname{section}{Section}{Sections}
\Crefname{table}{Table}{Tables}
\crefname{table}{Tab.}{Tabs.}

\newcommand{\ourmethodsimple}{\textsc{QuMoSeg-v2}\xspace} 
\newcommand{\ourmethod}{\textsc{QuMoSeg-v1}\xspace}

\usepackage[accsupp]{axessibility}  

\title{Quantum Motion Segmentation}
\author{Federica Arrigoni$^1$, Willi Menapace$^1$, Marcel Seelbach Benkner$^2$, \\ Elisa Ricci$^{1,3}$ and Vladislav Golyanik$^4$ \\ 
\quad  \\ 
\small $^1$ University of Trento \quad $^2$ University of Siegen \quad $^3$ Bruno Kessler Foundation \\ 
\small  $^4$  Max Planck Institute for Informatics, SIC}
\date{ }

\begin{document}

\maketitle

\begin{abstract}
Motion segmentation is a challenging problem that seeks to identify independent motions in two or several input images. This paper introduces the first algorithm for motion segmentation that relies on adiabatic quantum optimization of the objective function. The proposed method achieves {on-par performance with the state of the art} on problem instances which can be mapped to modern quantum annealers. 
\end{abstract}

\begin{figure}[b!] 
  \centering
  \begin{tabular}{ccccc} 
\small Ground truth & Our method & \textsc{Mode} \cite{ArrigoniPajdla19a} & \textsc{Synch} \cite{ArrigoniPajdla19b} & Xu et al. \cite{XuCheongAl19} \\
\includegraphics[width=0.175\linewidth]{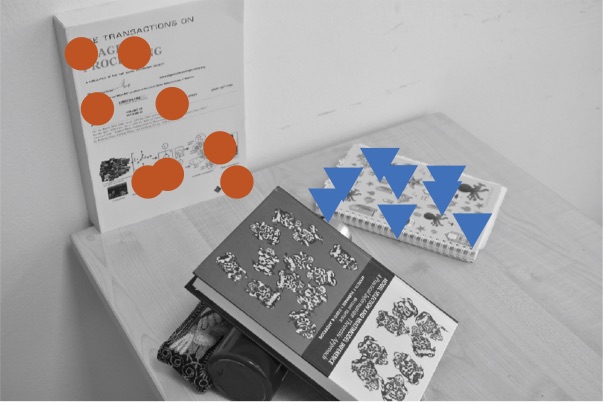} & \includegraphics[width=0.175\linewidth]{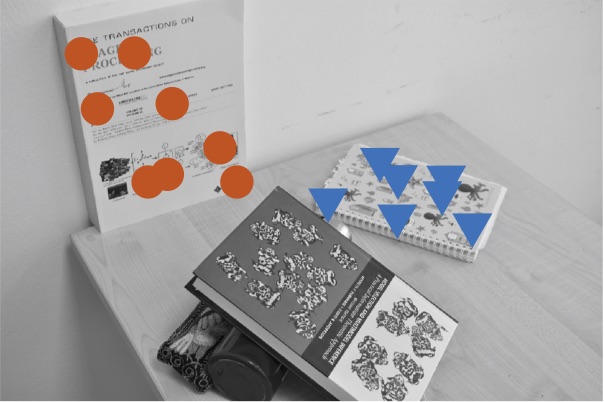} & \includegraphics[width=0.175\linewidth]{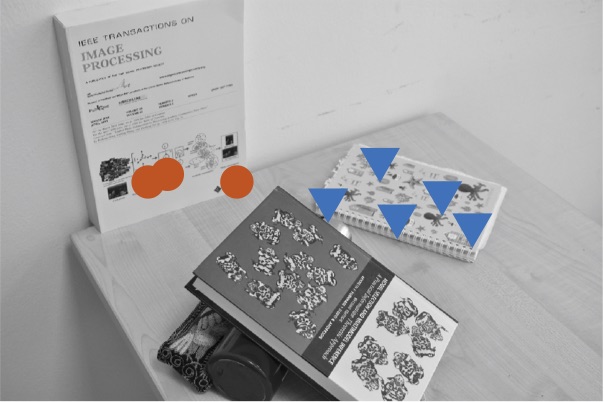} & \includegraphics[width=0.175\linewidth]{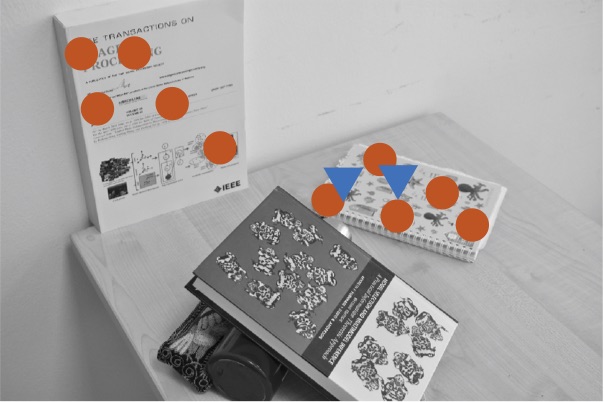} &
\includegraphics[width=0.175\linewidth]{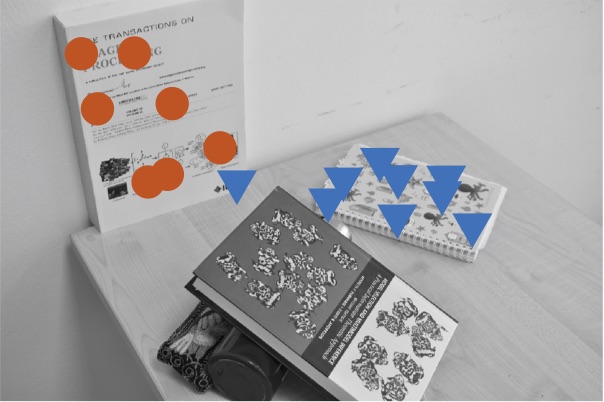}\\
\includegraphics[width=0.175\linewidth]{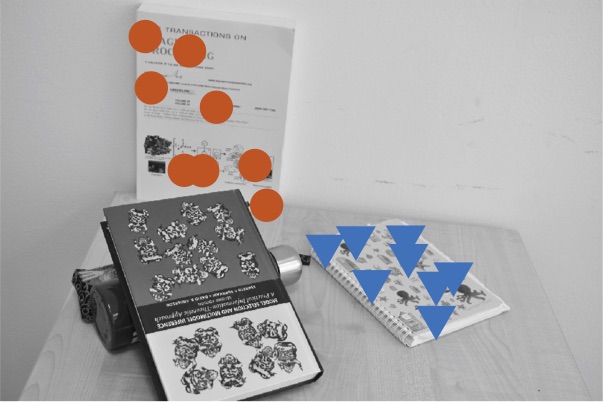} & \includegraphics[width=0.175\linewidth]{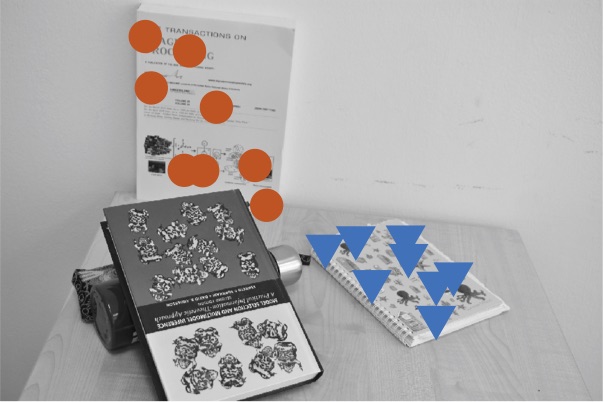} & \includegraphics[width=0.175\linewidth]{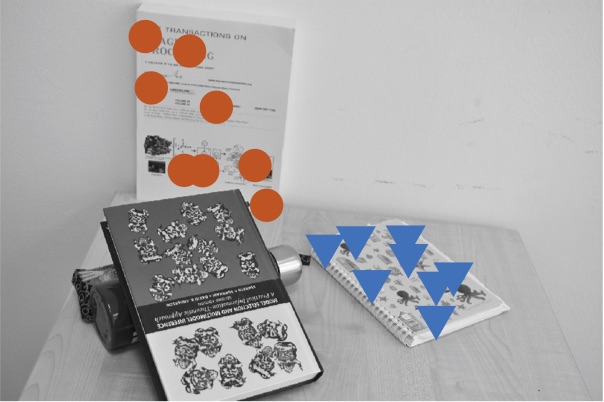} & \includegraphics[width=0.175\linewidth]{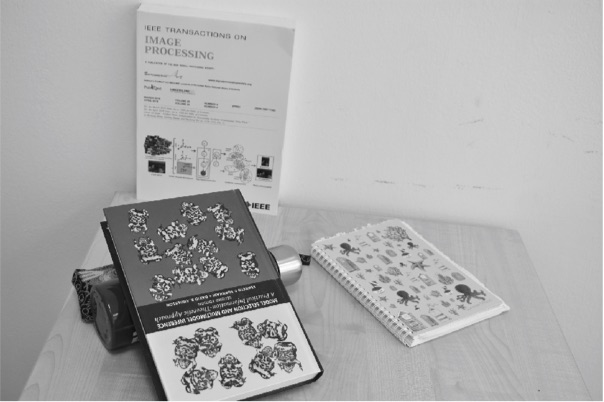}  & \includegraphics[width=0.175\linewidth]{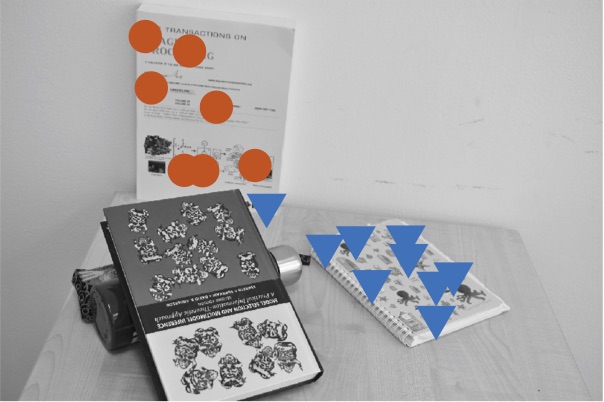} \\
\includegraphics[width=0.175\linewidth]{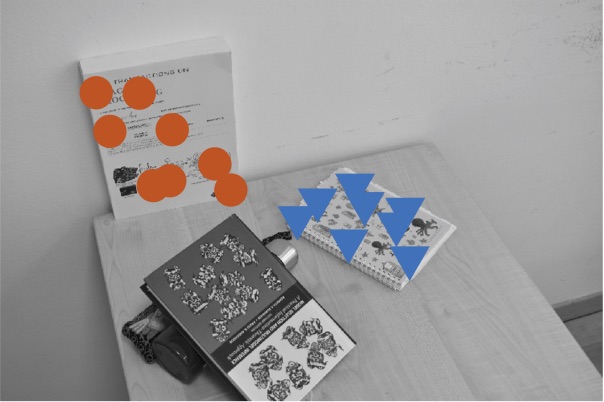} & \includegraphics[width=0.175\linewidth]{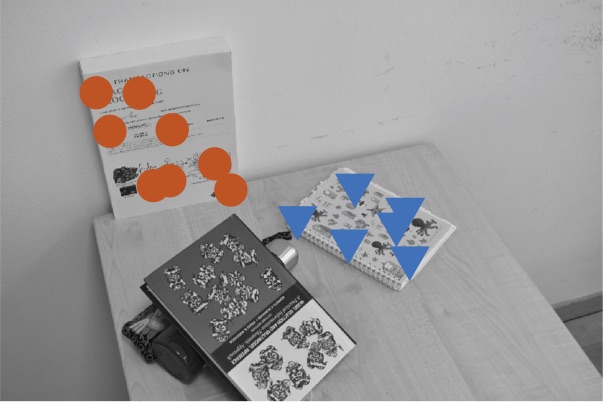} & \includegraphics[width=0.175\linewidth]{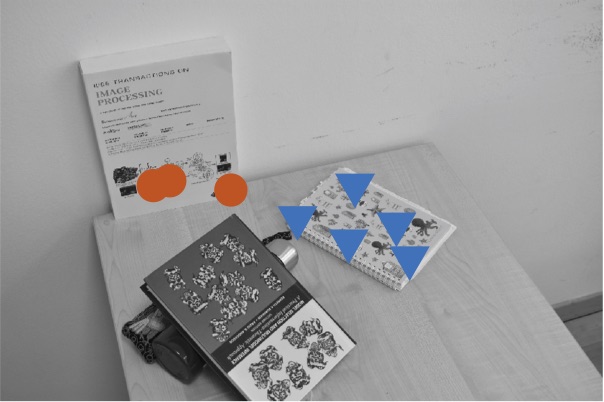} & \includegraphics[width=0.175\linewidth]{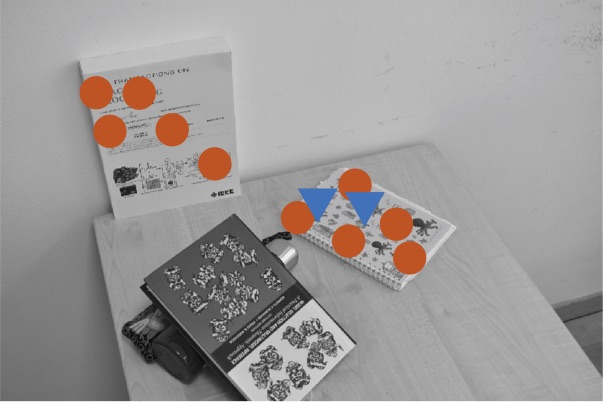} & \includegraphics[width=0.175\linewidth]{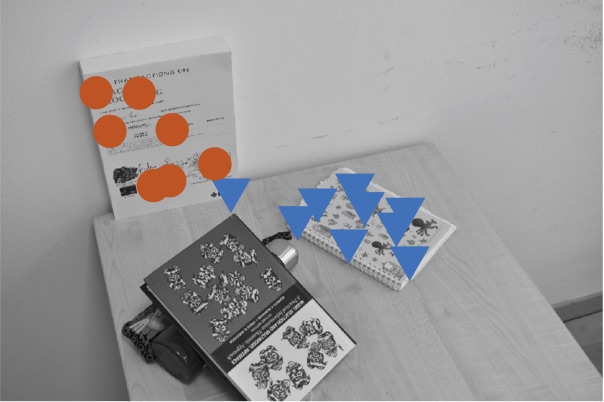}  \\
\end{tabular}
\caption{
Qualitative results on sample images from the new Q-MSEG dataset,  where each color (symbol) represents a distinct planar motion.
On average, the accuracies of our \ourmethod, \textsc{Mode} \cite{ArrigoniPajdla19a}, \textsc{Synch} \cite{ArrigoniPajdla19b} and Xu et al. \cite{XuCheongAl19} are $0.97$, $0.93$, $0.93$ and $0.89$, respectively, on problems with 96 qubits. 
In the shown example, our approach outperforms the competitors. See Tab.~\ref{tab_real} for further details. 
}
\label{fig:qualitative}
\end{figure}

\section{Introduction}\label{sec:introduction} 

Quantum computer vision is an emerging field. 
Recently, several classical problems were reformulated to enable quantum optimization, including recognition \cite{OMalley2018,Cavallaro2020} and matching tasks \cite{seelbach20quantum,BirdalGolyanikAl21}.
Promising results were shown so far, thus encouraging further research. 
Among the two existing paradigms for quantum computing, \textit{i.e.,} gate-based and adiabatic quantum computing (AQC), experimental realizations of AQC are already applicable to real-world problems, provided that the objective is given as a \emph{quadratic unconstrained binary optimization} (QUBO) problem. 
Thus, quantum annealing (QA)---which refers to not perfectly adiabatic implementations of AQC \cite{DWaveWebsite,grant2020adiabatic}---is an  experimental and promising technology for finding  solutions to combinatorial problems leveraging quantum mechanics \cite{Farhi2001,Denchev2016}. 
QA optimises objectives \emph{without relaxation} and obtains \emph{globally-optimal or low-energy solutions} with high probabilities.  Note that these important properties are hardly present in traditional methods, hence it is crucial to identify problems benefiting from this new class of machines.

In \cite{BirdalGolyanikAl21}
an AQC algorithm for \emph{permutation synchronization} is proposed, which finds cycle-consistent matches across a set of images or shapes, where the matches are given as permutation matrices. 
The recent survey \cite{ArrigoniFusiello19} discusses many synchronization problems already studied in the literature (\textit{e.g.,} rotation synchronization for structure from motion \cite{ChatterjeeGovindu13} or pose synchronization for point-set registration \cite{GovinduPooja14}). 
Notwithstanding, permutation synchronization  \cite{BirdalGolyanikAl21} is the only one that has been solved via quantum optimization so far. 

This paper advances the state of the art in quantum computer vision by bringing a new synchronization problem, \textit{i.e.,} \emph{motion segmentation}, into an AQC-admissible form; see  Fig.~\ref{fig:qualitative} for exemplary results. 
The task of motion segmentation \cite{SaputraMarkhamAl18} is to classify  points in multiple images into different motions, which is relevant in such applications as dynamic 3D reconstruction \cite{OzdenSchindlerAl10} or autonomous driving \cite{SabzevariScaramuzza16}. 
Observe that quantum formulations do not make sense for all problems, but for those, \textit{e.g.,} that include \emph{combinatorial} optimisation objectives, which are usually $\mathcal {NP}$-hard. 
Motion segmentation is identified to have a combinatorial structure, and, hence, is a meaningful candidate to leverage the advantages of the quantum processor. Bringing motion segmentation into an AQC is challenging as only problems in a QUBO form are admitted. Thus, we primarily focus on how to formulate motion segmentation as a QUBO.

Our work adopts the synchronization formulation of motion segmentation from \cite{ArrigoniPajdla19b}, which we carefully convert to a QUBO problem. 
This gives rise to the first variant of our quantum approach, named \ourmethod (from ``\textbf{QU}antum \textbf{MO}tion \textbf{SEG}mentation''): it works well in many practical scenarios but it can not manage large-scale problems since it is based on a \emph{dense} matrix.
For this reason, we also develop an alternative method based on a \emph{sparse} matrix which can solve larger problems, resulting in \ourmethodsimple: its derivation, however, requires additional assumptions, \textit{i.e.,} the knowledge of the number of points per motion.
In summary, our primary contributions are: 

\begin{itemize}\itemsep0em 
\item[1)] A new approach to motion segmentation that employs AQC (Sec.~\ref{sec:approach}); 
\item[2)] A new real dataset (Q-MSEG) for motion segmentation (Sec.~\ref{sec:experiments}). 
\end{itemize} 

We evaluated our approach on synthetic data, a new real dataset (Q-MSEG) and small problems sampled from the Hopkins benchmark \cite{TronVidal07}.
In our extensive experiments, our approach achieves competitive accuracy (close to or higher than competing methods) on problem instances which are mappable to the AQC of the latest generation, and demonstrates its high robustness to noise.
Due to the limits of current quantum hardware (that improves constantly and is far from maturity), our experiments are limited to small-scale data, as done also in previous work \cite{BirdalGolyanikAl21}. However, it is expected that progress in quantum hardware alongside with the ability to solve combinatorial problems without approximation will give practical advantages for large-scale problems in the future.

Our derivations share only a few  similarities with Birdal \textit{et al.}~\cite{BirdalGolyanikAl21}: In fact, bringing motion segmentation into a QUBO form requires more analytical steps compared to previous work on quantum synchronization \cite{BirdalGolyanikAl21} because binary matrices are less constrained than permutation matrices.
Moreover, our formulation for synchronization requires linearly-many variables in the number of input points, hence we are can handle more points compared to \cite{BirdalGolyanikAl21}. 
Our source code and data will be made publicly available.

\section{Background}\label{sec_background} 

Our work is inspired by Arrigoni and Pajdla \cite{ArrigoniPajdla19b} (reviewed in Sec.~\ref{sec:motion_segmentation}), where a convenient matrix representation is introduced for motion segmentation from pairwise correspondences. 
In Sec.~\ref{sec:approach} we will show how to rewrite such framework in terms of a QUBO, in order to enable adiabatic quantum optimization. In this respect, we report some preliminary notions on quantum computing in Sec.~\ref{sec:quantum_annealing}.

\begin{figure}[t!] 
  \centering
\includegraphics[width=0.495\linewidth]{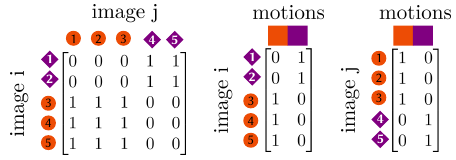} 
\includegraphics[width=0.495\linewidth]{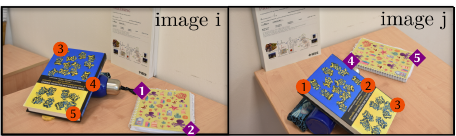} 
\caption{Matrix representation of motion segmentation. }
\label{fig:ij_motions}
\end{figure}

\subsection{Motion Segmentation}\label{sec:motion_segmentation} 

The objective of the motion segmentation problem is to group key-points in multiple images according to a number of motions. 
We use the following notation: $n$ is the number of images; $p_i$ is the number of key-points in image $i$; $p=\sum_{i=1}^n p_i$ is the total amount of key-points; $d$ is the number of motions (known by assumption).
We focus here on motion segmentation from pairwise correspondences \cite{ArrigoniPajdla19a,ArrigoniPajdla19b}, which can be addressed in two steps: 1) motion segmentation is addressed on different \emph{pairs} of images independently, which in turn can be done via multi-model fitting \cite{MagriFusiello14,MagriFusiello15,BarathMatas18}; 2) the results derived in the first step are globally combined, thus producing the required \emph{multi-frame} segmentation.

As shown in \cite{ArrigoniPajdla19b}, motion segmentation can be seen as a ``synchronization'' of binary matrices. Indeed, the result of motion segmentation in two images $i$ and $j$ can be represented as a matrix $Z_{ij} \in \{ 0,1\}^{ p_i \times p_j}$ as follows:
\begin{itemize}\itemsep0em 
\item  $[Z_{ij}]_{h,k}=1$ if point $h$ in image $i$ and point $k$ in image $j$ belong to the same motion;
    \item  $[Z_{ij}]_{h,k}=0$ otherwise.
\end{itemize}
The (known) binary matrix $Z_{ij}$ is referred to as the \emph{partial segmentation} or \emph{relative segmentation} of the pair $(i,j)$. 
It is a \emph{local} representation of segmentation, since it reveals which points in two different images belong to the same motion, but it does not reveal which motion it is with respect to other pairs. Similarly, our desired output can be represented as a matrix $X_i \in \{ 0,1\}^{ p_i \times d}$ as follows:
\begin{itemize}
    \item  $[X_i]_{h,k}=1$ if point $h$ in image $i$ belongs to motion $k$;
    \item  $[X_i]_{h,k}=0$ otherwise.
\end{itemize}
The binary matrix $X_i$ is called the \emph{total segmentation} or \emph{absolute segmentation} of image $i$. Observe that the number of rows is equal to the number of points while the number of columns is equal to the number of motions. Note also that it is a \emph{global} representation of segmentation since it reveals the membership of all points with respect to an absolute order of motions.
The notions of absolute and relative segmentations are illustrated in Fig.~\ref{fig:ij_motions}.

\begin{remark}
Let $\mathbf{m}_i$ be a vector of length $d$ such that $[\mathbf{m}_i]_h$ counts how many points in image $i$ belongs to motion $h$. Then the columns in $X_i$ sum to $\mathbf{m}_i$: 
\begin{equation}
       \vone_{p_i}^\tr X_i  = \mathbf{m}_i^\tr,
      \label{eq_known_di}
\end{equation}
where $\vone$ is a vector of ones (with length given as subscript).
Note also that the product $X_i^\tr X_i$ is a $d \times d$ diagonal matrix:
\begin{equation}
    X_i^\tr X_i = \diag(\mathbf{m}_i).
    \label{eq_diag_di}
\end{equation}
These simple properties will be exploited later.
\label{rem_diagonal}
\end{remark}

The connection between relative and absolute segmentations \cite{ArrigoniPajdla19b} is given by:
\begin{equation}
    Z_{ij} = X_i X_j^{\mathsf{T}}.
    \label{eq_consistency}
\end{equation}
Recall that the left side in the above equation is known whereas the right side is unknown. 
In general, there are multiple image pairs giving rise to an equation of the form \eqref{eq_consistency}, which can be conveniently represented as the edge set $\mathcal{E}$ of a graph $\mathcal{G} = (\mathcal{V}, \mathcal{E})$, where $\mathcal{V}=\{1, \dots, n\}$ denotes the vertex set. In other terms, each vertex represents an image and an edge is present between two vertices if and only if the relative segmentation of that image pair is available.
Thus motion segmentation can be cast to the task of recovering $X_1, \dots, X_n$ starting from $Z_{ij}$ with $(i,j) \in \mathcal{E}$, such that \eqref{eq_consistency} is satisfied. This is also called \emph{synchronization} \cite{ArrigoniFusiello19}.

\subsection{Adiabatic Quantum Optimization}
\label{sec:quantum_annealing} 

Modern AQC can solve \emph{quadratic unconstrained binary optimization} (QUBO) problems of the form 
\begin{equation}
\min_{\mathbf{y} \in \mathcal{B}^k } \mathbf{y}^\mathsf{T} Q \mathbf{y} + \mathbf{s}^\tr \mathbf{y}, 
    \label{eq_qubo}
\end{equation}
where $\mathcal{B}^k$ denotes the set of binary vectors of length $k$, $Q \in \mathbb{R}^{k \times k}$ is a real symmetric matrix and $\mathbf{s} \in  \mathbb{R}^k$. 
Note that optimization is performed over binary variables. 
QUBO problems are $\mathcal{NP}$-hard. 
An optimization problem with hard constraints of the form $ A_i \mathbf{y}= \mathbf{b_i}$ can be converted to a QUBO with soft constraints \cite{BirdalGolyanikAl21}, in which linear terms weighted by multipliers $\lambda_i$ rectify $Q$ and $\mathbf{s}$: 
\begin{equation}
\min_{\mathbf{y} \in \mathcal{B}^k } \mathbf{y}^\mathsf{T} Q \mathbf{y} + \mathbf{s}^\tr \mathbf{y} + \sum_i \lambda_i\,|| A_i \mathbf{y} - \mathbf{b_i} ||^2 . 
\label{eq_qubo_linear_lambda}
\end{equation}
As \eqref{eq_qubo} does not easily allow including high-level constraints (\textit{e.g.,} on a matrix rank), most computer vision problems cannot be easily posed in a QUBO form. 
Latest research thus focuses on finding such  formulations \cite{SeelbachBenkner2021,BirdalGolyanikAl21}. 

AQC interprets $\mathbf{y}$ as a measurement result of $k$ qubits, and optimization of \eqref{eq_qubo} is performed on AQC not in the binary vector space but a ``lifted'', $2^{k}$-dimensional space of $k$ qubits, taking advantage of quantum-mechanical effects like superposition and entanglement. 
In contrast to a classical bit which can be either in  state $0$ or $1$ at a time, a qubit $\ket{q} = \alpha \ket{0} + \beta  \ket{1}$ can take any state fulfilling 
$\alpha, \beta \in \mathbb{C}$ and  $|\alpha|^2 + |\beta|^2 = 1$. 
Once a QUBO form is known, it is first passed to a \emph{minor embedding} algorithm such as Cai \textit{et al.}~\cite{Cai2014}. 
Its purpose it to find a mapping of a QUBO problem \eqref{eq_qubo} defined in terms of qubits---which in the following we call \textit{logical} (\textit{i.e.,} mathematical models)---to an AQC with \textit{physical} qubits (\textit{i.e.,} hardware realizations of the mathematical models). 
This step is necessary for most problems except the smallest ones, as the qubit connectivity pattern encoded in $Q$ is not natively supported by the hardware \cite{Dattani2019} 
and several repeated physical qubits, building a  \textit{chain}, are required to represent a single logical qubit during quantum annealing.

After the initialisation in a problem-independent state, AQC is transitioning from the initial solution to the solution of the target problem in the QUBO form, \textit{i.e.,} one says that the system evolves its state (or an \textit{annealing} is taking place) according to the rules of quantum mechanics \cite{Farhi2001,McGeoch2014}. 
The notion \textit{adiabatic} refers to how this transition  happens in the ideal case, namely obeying the adiabatic theorem of quantum mechanics \cite{BornFock1928}. 
The remaining steps of an AQC algorithm are: 1) Sampling; 2)  Unembedding; 3) Bitstring selection and 4) Solution  interpretation \cite{SeelbachBenkner2021}. 
QA is probabilistic in nature, and a globally-optimal measurement can be obtained with specific success probabilities. 
Thus, multiple annealings are required to reach a satisfactory result (QUBO \textit{sampling}). 
The number of repetitions can vary by orders of magnitude depending on the probability to measure an optimal solution, the problem size and the minor embedding. 
Each sample is measured and unembedded, \textit{i.e.,} returned in terms of the logical qubit measurements. 
Next, one or several samples are chosen as the final solution(s), and the most common  criterion is the lowest energy (\textit{i.e.,} minimal cost) over all samples. 
An interested reader can further refer to McGeoch \cite{McGeoch2014}.

\section{Our Quantum Approach}\label{sec:approach} 

At the core of our approach is a QUBO formulation of motion segmentation. This permits---for the first time in the literature---to solve the segmentation task via quantum optimization. 
Bringing it into a QUBO form is not straightforward and requires more analytical steps compared to previous work on quantum synchronization \cite{BirdalGolyanikAl21}, as will be clarified later.
We propose two methods: the first one is based on a dense matrix, and, hence, can require an increased embedding size (Sec.~\ref{sec_dense_matrix}); the second one, instead, is based on a sparse matrix but relies on  additional assumptions (Sec.~\ref{sec_simplified}).
Both variants are principled (\textit{i.e.,} are equivalent to synchronization) and can solve real-world problems (see Sec.~\ref{sec:experiments}).

\subsection{QuMoSeg-v1}
\label{sec_dense_matrix}

As explained in Sec.~\ref{sec:motion_segmentation}, motion segmentation can be posed as computing absolute segmentations (\textit{i.e.,} $X_1, \dots, X_n$) starting from pairwise segmentations $Z_{ij}$ with $(i,j) \in \mathcal{E}$ such that $Z_{ij}=X_i X_j^\tr $. 
In the presence of noise, the task is to solve the following optimization problem:
\begin{equation}
\begin{gathered}
\min_{X_1, \dots X_n} \sum_{(i,j) \in \mathcal{E}}   || Z_{ij} -  X_i X_j^{\mathsf{T}} ||_F^2, \\ \text{s.t.} \ \vect(X_i) \in \mathcal{B}^{p_i}, \quad X_i \vone_{d} = \vone_{pi} \quad \forall i=1, \dots, n,  
\end{gathered}
\label{cost_synch}
\end{equation}
where $\mathcal{B}^k$ denotes the set of binary vectors of length $k$, and $\vect(\cdot)$ denotes the vectorization operator that transforms a matrix into a vector by stacking the columns one under the other. Recall that $X_i$ has size $p_i \times d$, so $\vect(X_i)$ has length $dp_i$ and should be a binary vector.
The constraint  $ X_i \vone_{d} = \vone_{pi} $ means that each row sums to 1. Indeed, each row in $X_i$ has exactly one entry equal to 1 (whereas all other entries are zero), meaning that each point should belong to exactly one motion. 
The cost in \eqref{cost_synch} measures, for each image pair, the discrepancy (in the Frobenius norm sense) between the input relative segmentation (\textit{i.e.,} $Z_{ij}$) and the relative segmentation derived from the sought absolute segmentations (\textit{i.e.,} $X_i X_j^{\mathsf{T}}$). It is also known as the \emph{consistency error} \cite{ArrigoniFusiello19}. 

\begin{proposition}
Problem \eqref{cost_synch} is equivalent to
\begin{equation}
\begin{gathered}
\max_{X_1, \dots X_n} \sum_{(i,j) \in \mathcal{E}}  \trace( X_i^\tr (2Z_{ij} - \vone_{p_i \times p_j}) X_j ) \\ 
\text{s.t.} \ \vect(X_i) \in \mathcal{B}^{p_i}, \quad X_i \vone_{d} = \vone_{pi} \quad \forall i=1, \dots, n
\end{gathered}
\label{cost_trace_NEW}
\end{equation}
where $\vone_{p_i \times p_j} $ denotes a $p_i \times p_j$ matrix of ones. 
\label{prop_synch_trace_NEW}
\end{proposition}
\begin{proof}
For simplicity of exposition, we drop constraints and focus on the cost function itself. By computation and exploiting \eqref{eq_diag_di}, we obtain: 
\begin{equation}
\begin{gathered}
 || Z_{ij} -  X_i X_j^{\mathsf{T}} ||_F^2 =  
  \trace(Z_{ij}^{\mathsf{T}}Z_{ij}) + \trace(X_jX_i^\tr X_i X_j^\tr) -2 \trace(Z_{ij}^\tr X_i X_j^\tr) = \\
  \trace(Z_{ij}^{\mathsf{T}}Z_{ij})   
  + \trace(\diag(\mathbf{m}_i) \diag(\mathbf{m}_j) )
  -2    \trace( X_i^\tr Z_{ij} X_j )  = \\
  =\trace(Z_{ij}^{\mathsf{T}}Z_{ij}) + 
  \mathbf{m}_i^\tr \mathbf{m}_j 
  -2    \trace( X_i^\tr Z_{ij} X_j ) .
\end{gathered}
\label{eq_synch_trace_NEW} 
\end{equation}
Note that the first term is constant, for it depends on the input $Z_{ij}$ only, hence it can be ignored in the optimization. 
As for the second term, using \eqref{eq_known_di}, we get: 
\begin{equation}
\begin{gathered}
 \mathbf{m}_i^\tr \mathbf{m}_j  = \trace( \mathbf{m}_i  \mathbf{m}_j^\tr) = \trace(X_i^\tr \vone_{p_i} \vone_{p_j}^\tr X_j )    
 = \trace (X_i^\tr  \vone_{p_i \times p_j} X_j). 
 \end{gathered}
\end{equation}
Hence, the optimization in \eqref{cost_synch} is equivalent to \eqref{cost_trace_NEW}. 

\end{proof}

We now rewrite \eqref{cost_trace_NEW} in a compact form via the following notation, where all the measures/unknowns are grouped into block-matrices $X \in \mathbb{R}^{p{\times}d}$ and $Z\in \mathbb{R}^{p{\times}p}$:
\begin{equation}
X=
\begin{bmatrix}
X_{1} \\
X_{2} \\
\dots \\
X_{n}
\end{bmatrix}, \quad
Z = 
\begin{bmatrix}
0 & Z_{12} & \dots & Z_{1n} \\
Z_{21} & 0 & \dots & Z_{2n} \\
\dots &  &  & \dots \\
Z_{n1} & Z_{n2} & \dots &  0
\end{bmatrix}.
\label{eq_X_Z}
\end{equation}

\begin{proposition}
Problem \eqref{cost_synch} is equivalent to
\begin{equation}
\begin{gathered}
\max_{X} \vect(X)^\tr ( I_{d \times d}\otimes(2Z - \vone_{p \times p})   ) \vect(X), \\ \text{s.t.} \ \vect(X) \in \mathcal{B}^{dp}, \quad X \vone_{d} = \vone_{p}.
\end{gathered}
\label{eq_cost_matrixform}
\end{equation}
\label{prop_dense_matrix}
\end{proposition}

\begin{proof}
Let us define $W=2Z - \vone_{p \times p}$. Using \eqref{eq_X_Z}, we get:
\begin{equation}
\small 
\sum_{(i,j) \in \mathcal{E}}  \trace( X_i^\tr (2Z_{ij} - \vone_{p_i \times p_j}) X_j )  =   \trace( X^\tr W X ).  
\end{equation}
The above equation can be further simplified by exploiting properties of the trace operator\footnote{For any matrices $A$, $B$ of proper dimensions we have: $\trace(A^\tr B) = \vect(A)^\tr \vect(B) $.} and the Kronecker product\footnote{For any matrices $A,B,Y$ of proper dimensions, the Kronecker product \cite{LiuTrenkler08} satisfies: $ \vect(AYB) = (B^\tr \otimes A) \vect(Y) $.\label{footnote_kron}} denoted by $\otimes$, resulting in:
\begin{equation}
\small 
\trace( X^\tr W X ) =  \vect(X)^\tr \vect(WX) =  \vect(X)^\tr (I_{d \times d} \otimes W) \vect(X), 
\end{equation}
where $I_{d \times d}$ denotes the $d \times d$ identity matrix. 
Hence, the objective function   \eqref{eq_cost_matrixform} is the same as 
\eqref{cost_trace_NEW}, which in turn is equivalent to \eqref{cost_synch}, as shown in Prop.~\ref{prop_synch_trace_NEW}.
As for constraints, it is easy to see that
$\vect(X_i) \in \mathcal{B}^{p_i}$ and $ X_i \vone_{d} = \vone_{pi}$ 
translate into $\vect(X) \in \mathcal{B}^{dp} $ and $ X \vone_{d} = \vone_{p}$, when considering all the unknowns simultaneously as stored in the block-matrix $X$. Hence we get the thesis.  
\end{proof}

\begin{corollary}
 Problem 
\eqref{cost_synch} can be mapped into a QUBO problem \eqref{eq_qubo_linear_lambda} 
\begin{equation}
\begin{gathered}
\min_{\mathbf{y} \in \mathcal{B}^k } \mathbf{y}^\mathsf{T} Q \mathbf{y} 
+ \lambda_1 ||  A \mathbf{y} - \mathbf{b} ||^2, \end{gathered}
\label{eq_QUBO_final_dense}
\end{equation}
where 
$ Q = -I_{d \times d}\otimes(2Z - \vone_{p \times p}) $,    $\ \mathbf{y} = \vect(X)$,    $ \ A = (\vone_{d}^\tr \otimes I_{p\times p})$,    $\ \mathbf{b} = \vone_{p}$.
\end{corollary}

\begin{proof}
Problem \eqref{eq_cost_matrixform} is the maximization of a quadratic cost function with binary variables and linear constraints, hence definitions of $Q$ and $\mathbf{y}$ are immediate. The linear constraints $ X \vone_{d} = \vone_{p} $  can be easily mapped into the canonical form $A \mathbf{y} = \mathbf{b} $ via properties of Kronecker product\footref{footnote_kron} and vectorization.  
\end{proof}

\begin{remark}
$Q=-I_{d \times d}\otimes(2Z - \vone_{p \times p})$ is symmetric and its size is $dp \times dp$, where $d$ is the number of motions and $p$ is the total amount of points over all images. Note also that the size of the optimization variable $\mathbf{y}$ is $dp$, so it scales linearly with the number of points (assuming $d \ll p$, which is usually the case in practice). 
\label{remark_nopoints}
\end{remark}

To summarize, the synchronization formulation for motion segmentation can be cast to a QUBO, thus enabling adiabatic quantum optimization. This gives rise to the first variant of our approach, which is called \ourmethod.

\subsection{QuMoSeg-v2}
\label{sec_simplified}

Note that the block-matrix $Z$ storing all partial segmentations is \emph{sparse}, \textit{i.e.,} most of its entries are zero (see Fig.~\ref{fig:ij_motions} for an example of a partial segmentation). 
However, the matrix $ 2Z - \vone_{p \times p}$ (which appears in the definition of $Q$ in \eqref{eq_QUBO_final_dense}) is \emph{dense} (it has only -1 or +1 as possible entries). 
This may result in increased embedding size, which is undesirable in practice. 
This observation motivates the need for an alternative method based on a sparse matrix, which is explored here. This comes at the price of having additional assumptions, as shown below.

\begin{proposition}
Let us assume that the amount of points per motion is known in each image, namely $\mathbf{m}_i$ is known $\forall i=1, \dots, n$, where the $d$-length vector $\mathbf{m}_i$ is defined in Remark \ref{rem_diagonal}. 
Then, Problem \eqref{cost_synch} is equivalent to
\begin{equation}
\begin{gathered}
\max_{X_1, \dots X_n} \sum_{(i,j) \in \mathcal{E}}  \trace( X_i^\tr Z_{ij} X_j ), \\ 
\text{s.t.} \ \vect(X_i) \in \mathcal{B}^{p_i}, \quad X_i \vone_{d} = \vone_{pi}, \quad  \vone_{p_i}^\tr X_i  = \mathbf{m}_i^\tr  \quad \forall i=1, \dots, n. 
\end{gathered}
\label{cost_trace}
\end{equation}
\label{prop_synch_trace}
\end{proposition}
\begin{proof}
The starting point is Eq.~\eqref{eq_synch_trace_NEW}, which is copied here as a reference:
\begin{equation}
\begin{gathered}
 || Z_{ij} -  X_i X_j^{\mathsf{T}} ||_F^2 = \trace(Z_{ij}^{\mathsf{T}}Z_{ij}) + 
  \mathbf{m}_i^\tr \mathbf{m}_j 
  -2    \trace( X_i^\tr Z_{ij} X_j ) .
\end{gathered}
\end{equation}
As already observed, the first term is constant, for it depends on $Z_{ij}$ only; hence, it can be ignored in the optimization. 
Also the second term is constant, since $\mathbf{m}_i$ is known for each image $i$ \emph{by assumption}. 
Thus, the objective function \eqref{cost_synch} is equivalent to \eqref{cost_trace}. 
Observe that we should add extra constraints to take into account our additional assumptions, which force the amount of points per motion to be equal to some predefined values in every image, namely  $      \vone_{p_i}^\tr X_i  = \mathbf{m}_i^\tr$ (see also Remark \ref{rem_diagonal}). Hence we get the thesis.  
\end{proof}

\begin{remark}
Note that the knowledge of the number of points per motion was indeed essential in the proof of Prop.~\ref{prop_synch_trace}. In other terms, the synchronization problem \eqref{cost_synch} \textbf{is not equivalent} to \eqref{cost_trace} without such an assumption.  
\label{remark_no_equivalent}
\end{remark}

\begin{proposition}
If the amount of points per motion is known for each image, then Problem \eqref{cost_synch} is equivalent to
\begin{equation}
\begin{gathered}
\max_{X} \,  
 \vect(X)^\tr (I_{d \times d} \otimes Z) \vect(X), \\ \text{s.t.} \ \vect(X) \in \mathcal{B}^{dp}, \quad X \vone_{d} = \vone_{p}, \quad K X=M,  
\end{gathered}
\label{cost_trace_final}
\end{equation}
where $K$ and $M$ are defined as follows:
\begin{equation}
K= \operatorname{diag}(\vone_{p_1}^\tr, \vone_{p_2}^\tr, \hdots, \vone_{p_n}^\tr) , \quad 
M=
\begin{bmatrix}
\mathbf{m}_{1}^\tr &
\mathbf{m}_{2}^\tr &
\dots &
\mathbf{m}_{n}^\tr
\end{bmatrix}^\tr.
\label{eq_system_simplifiedass}
\end{equation} 
\label{prop_synch_trace_general}
\end{proposition}
\begin{proof}
Problem \eqref{cost_trace} can be easily turned into the form \eqref{cost_trace_final} following the same reasoning as in the proof of Prop.~\ref{prop_dense_matrix}. Concerning the additional constraints, it is easy to see that $KX=M$ is just a compact way of storing all the equations of the form $\vone_{p_i}^\tr X_i  = \mathbf{m}_i^\tr $ simultaneously.  
\end{proof}

\begin{corollary}
If the amount of points per motion is known for each image, then Problem \eqref{cost_synch} can be mapped into a QUBO problem of the form \eqref{eq_qubo_linear_lambda}, namely: 
\begin{equation}
\begin{gathered}
\min_{\mathbf{y} \in \mathcal{B}^{dp} } \mathbf{y}^\mathsf{T} P \mathbf{y} 
+ \lambda_2 ||  A \mathbf{y} - \mathbf{b} ||^2 
+ \lambda_3 ||  E \mathbf{y} - \mathbf{f} ||^2,     \end{gathered}
  \label{eq_QUBO_final}
\end{equation}
    \text{where} $ 
    P = - I_{d \times d} \otimes Z, \ 
 \mathbf{y} = \vect(X), \
    A = (\vone_{d}^\tr \otimes I_{p\times p}), \
    \mathbf{b} = \vone_{p}, \ \mathbf{f}=\vect(M),  $ and $ 
    E= (I_{d\times d} \otimes K)   $.
\label{corollary_QUBO}
\end{corollary}

\begin{proof}
Problem \eqref{cost_trace_final} is the maximization of a quadratic cost with binary variables and linear constraints, hence definitions of $P$ and $\mathbf{y}$ are immediate. The linear constraints $ X \vone_{d} = \vone_{p} $ and $KX = M$ can be mapped into the forms $A \mathbf{y} = \mathbf{b} $ and $E \mathbf{y} = \mathbf{f} $ via properties of Kronecker product\footref{footnote_kron} and vectorization.  
\end{proof}

In summary, we derived an alternative QUBO formulation for motion segmentation with a \emph{sparse} matrix: $P$ matrix defined in \eqref{eq_QUBO_final} is sparse as it inherits the same sparsity pattern as $Z$. This was possible \textbf{under simplified assumptions}, \textit{i.e.,} the knowledge of the amount of points per motion in all images. 
This gives rise to the second variant of our approach, which is called \ourmethodsimple. 

\paragraph{Execution on AQC.} 
Once the QUBO for the target data is known, namely  \eqref{eq_QUBO_final_dense} for \ourmethod or  \eqref{eq_QUBO_final} for \ourmethodsimple, motion  segmentation can be solved via adiabatic quantum  optimization, as described in  Sec.~\ref{sec:quantum_annealing}.

\input{related_work}

\section{Experiments}\label{sec:experiments} 

In this section, we report experimental results on synthetic scenarios, {a small dataset derived from the Hopkins benchmark \cite{TronVidal07} and a new real dataset}.  
We evaluate our method on D-Wave's AQC of the latest generation,  \textit{i.e.,} D-Wave Advantage4.1 (Adv4.1), which is an AQC of the Pegasus architecture. 
It contains ${\approx}5k$ physical qubits organized in cells of $24$ qubits; each qubit is coupled to $15$ other qubits; the total number of couplers is ${\approx}40k$ \cite{Boothby2020}. Adv4.1 operates at temperatures below $17mK$ and is accessed remotely via Ocean \cite{OceanTools2021}. We run all experiments with $20{\mu}s$ annealing time (no pause). 
The total QPU runtime of our experiments including  overheads amounts to over $15$ minutes (over $10^6$  obtained samples). 
We sample $1k$ times and take the lowest-energy sample as the result.

\textbf{Evaluation Methodology.} 
Since our method is the first quantum approach to segmentation with two-frame correspondences, we compare it with traditional approaches (\textit{i.e.,} not operating on AQC), namely  \textsc{Mode} \cite{ArrigoniPajdla19a} and \textsc{Synch} \cite{ArrigoniPajdla19b}, whose code is available online \cite{motionsegcode}. Both our approach and the competitors take the same input, namely a set of pairwise segmentations represented as a $dp{\times}dp$ block-matrix $Z$ -- see \eqref{eq_X_Z} -- and they compute the absolute segmentations either in the form of a $p{\times}d$ matrix $X$ or a vector $\mathbf{y}=\vect(X)$ of length $dp$. 
In order to compare a given solution $\mathbf{y} \in \mathcal{B}^{dp}$ with the ground-truth solution named $\mathbf{y}_{gt} \in \mathcal{B}^{dp}$, we use the \emph{accuracy} $\mu$\footnote{Other measures can be considered with similar results, such as the misclassification error, which is widely adopted in motion segmentation.}, as done also in \cite{BirdalGolyanikAl21}. It is defined as the number of correct entries over the total amount of entries, namely
\begin{equation}
    \mu = 1 - \mathcal{H}(\mathbf{y}_{gt}, \mathbf{y})/{dp}, 
\end{equation}
where $\mathcal{H}$ denotes the Hamming distance.

 \begin{figure}[!t] 
  \centering
    \subfloat[Accuracy (max value 1) for several methods on synthetic data versus input noise. \label{fig:synthetic}]{
\includegraphics[width=0.30\linewidth]{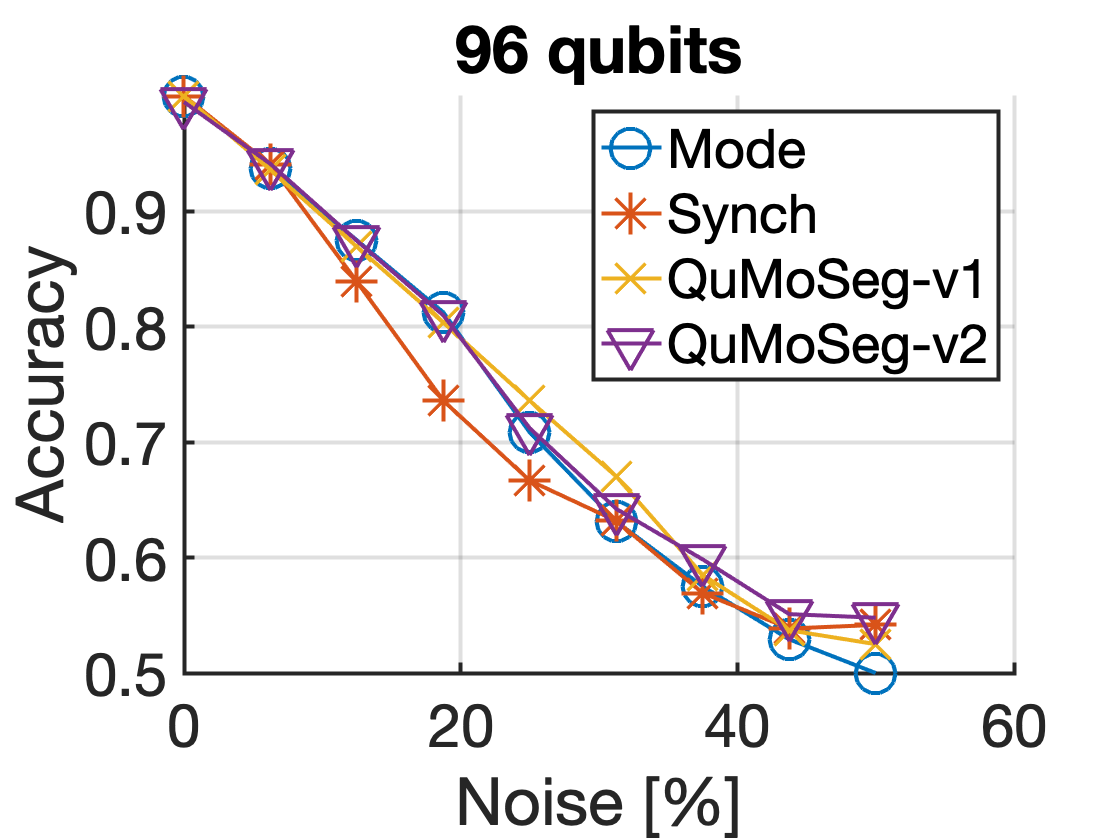} \quad \quad 
\includegraphics[width=0.30\linewidth]{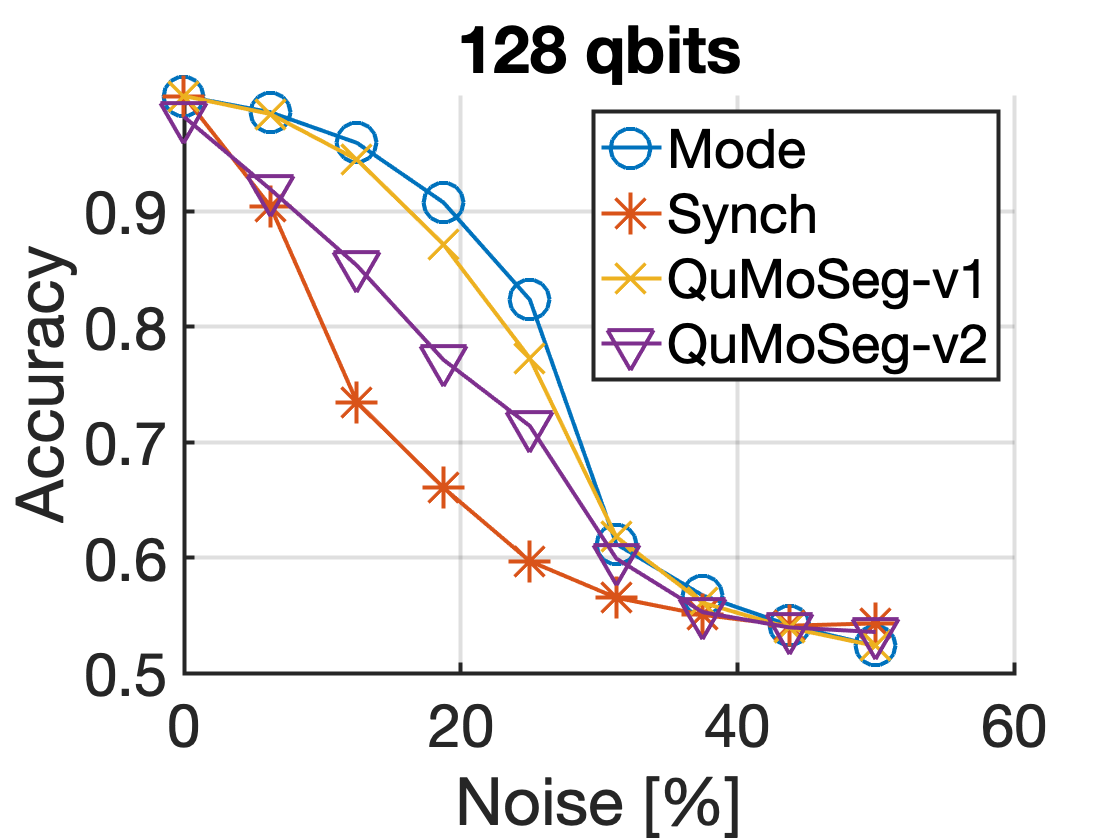} 
} \quad 
    \subfloat[Average maximum chain length (left), average number of physical qubits (middle) and probability of finding a solution (right) versus number of logical qubits in Q-MSEG dataset. \label{fig:chain_phys}]{
    \includegraphics[width=0.32\linewidth]{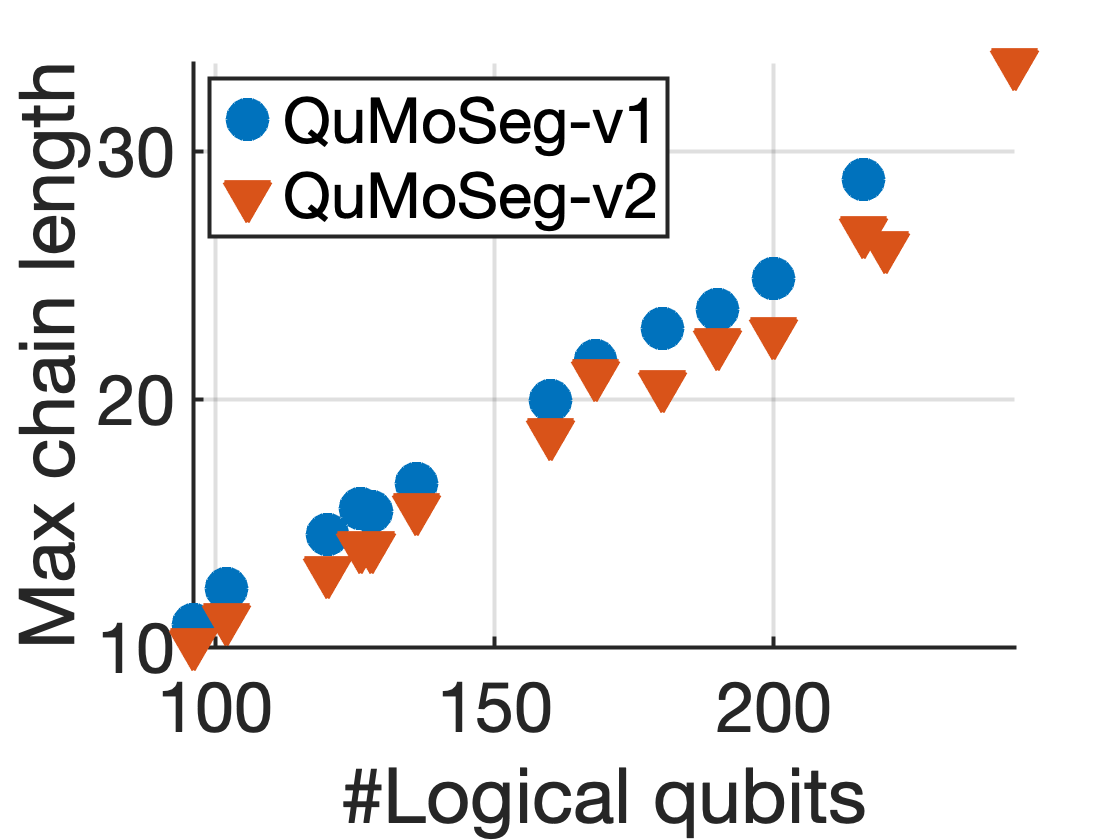}  
\includegraphics[width=0.32\linewidth]{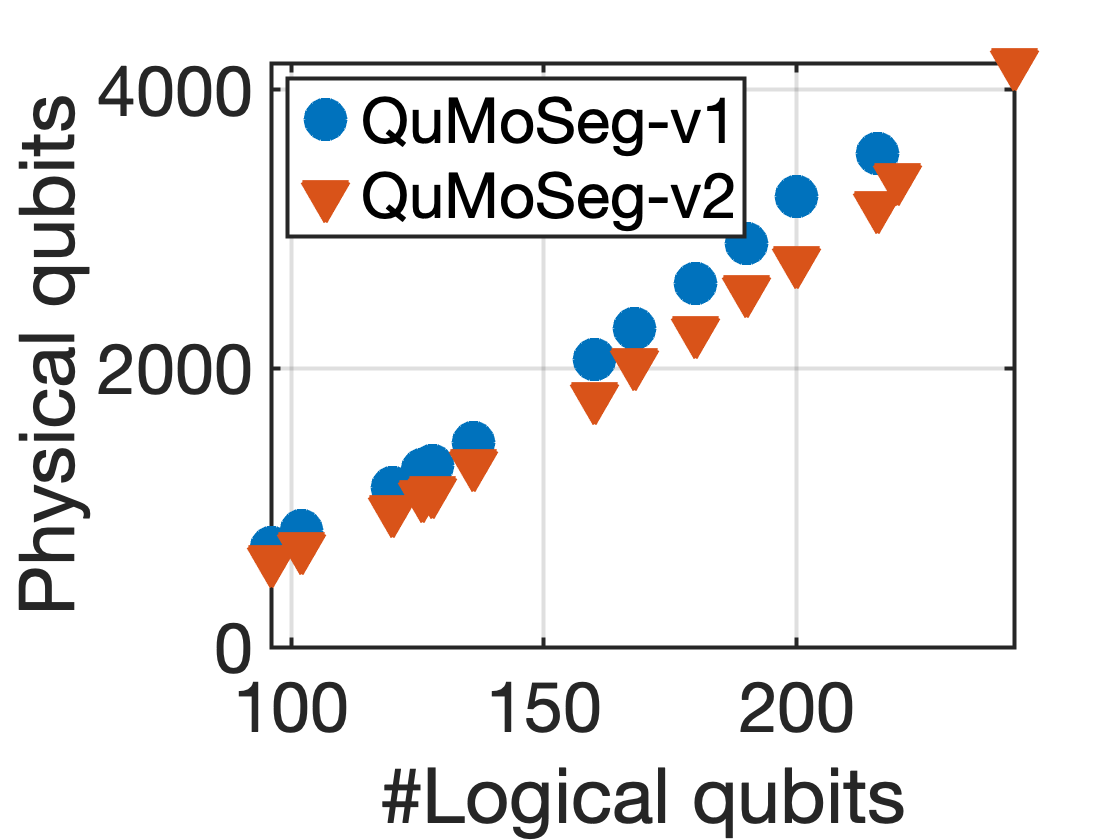}
\includegraphics[width=0.32\linewidth]{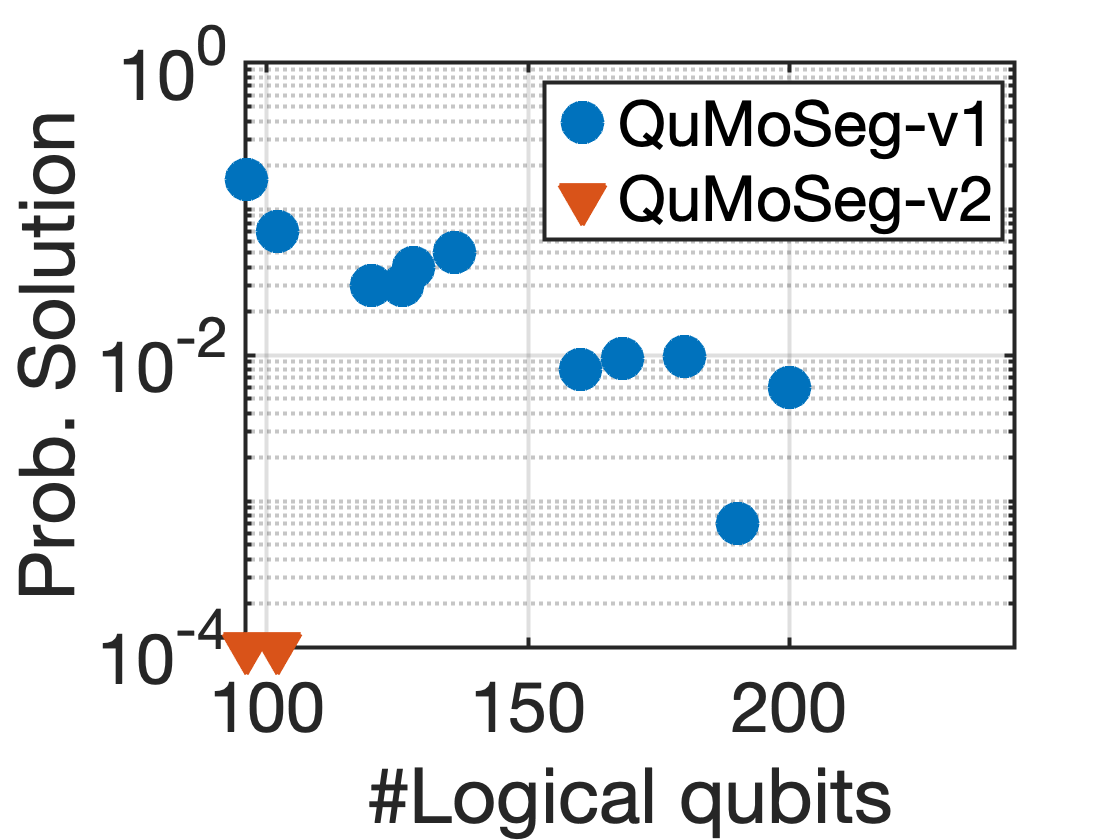} 
    }
\caption{
Results of synthetic experiments (a) and of real experiments on Q-MSEG (b).
}
\end{figure}

\textbf{Experiment with Synthetic Noise.}
We consider 20 synthetic configurations with 3 or 4 images, 2 motions and 16 points per image, resulting in 96 or 128 qubits. In order to create the ground-truth segmentation $X$, each point is assigned to a motion which is randomly chosen among the two available ones. All the analyzed methods can solve these problems with accuracy $\mu=1$. Next, we derive pairwise segmentations from \eqref{eq_consistency} and we systematically inject into them increasing amounts of noise:
in each pairwise segmentation we switch the motion of a percentage of points ranging from $0\%$ to $50\%$ (meaning that half of the points are corrupted). 
The results of this experiment are shown in Fig.~\ref{fig:synthetic}, which reports the accuracy (averaged over 20 problem instances) for the analyzed methods. 
\ourmethod is certainly on par with the best traditional approach (namely \textsc{Mode} \cite{ArrigoniPajdla19b}). 
Concerning \ourmethodsimple, it is worth noting that is starts from a similar formulation to \textsc{Synch} \cite{ArrigoniPajdla19b}, but it solves the problem \emph{without} relaxation: working over binary variables without relaxation in combination with AQC significantly improves the results. There are no significative differences between the two variants of our method on the smallest scenario with 96 qubits.

\textbf{Experiments on Q-MSEG Dataset.} 
Due to the lack of small-scale datasets usable for our scenario (\textit{i.e.,} whose size is manageable by current AQC),
we generate a new dataset for motion segmentation with ground-truth annotations, which comprises six images depicting three planar objects captured from diverse viewpoints. We focus on planar motions (where the underlying model is the homography) since they can be recovered from a lower amount of points than general motions\footnote{At least four points are needed to estimate a homography, whereas at least seven points are required for the fundamental matrix \cite{HartleyZisserman04}.} 
(where the corresponding model is the fundamental matrix), hence allowing more flexibility in the experimental setting. 
Each object is manually annotated with respectively 10, 11 and 12 keypoints, selected on highly-textured locations. 
Following \cite{BirdalGolyanikAl21,Wang2018MultiimageSM}, we extract features corresponding to each keypoint using an AlexNet model \cite{krizhevsky2012alexnet} pretrained on ImageNet. 
Several motion segmentation problems are derived from the dataset: we consider 14 different choices for the number of points/images/motions (see supplementary), and we sample 20 problem instances for each configuration, resulting in 280 problem instances in total. 
For each problem instance, we perform feature matching on every image pair via nearest neighbor search, where the cosine distance is used as the distance metric. Then, different motions are identified in each image pair by fitting multiple homographies to correspondences \cite{MagriFusiello15}. This defines a set of pairwise segmentations which represent the input of \textsc{Mode} \cite{ArrigoniPajdla19a}, \textsc{Synch} \cite{ArrigoniPajdla19b} and our quantum methods. In order to enrich the evaluation, we evaluate here also the method by Xu et al.~\cite{XuCheongAl19} (with public code \cite{xuncode}), although it makes different assumptions on the input (i.e., it requires multi-frame trajectories, which were computed by us following the same procedure as in \cite{ArrigoniPajdla19a,ArrigoniPajdla19b}).
For completeness, we also optimize the objective functions of our approach with a simulated annealing (SA) solver \cite{Kirkpatrick1983,DWave_neal}: 
while in general it performs on par with QA for smaller problems, it provides an indication on the QPU's solution accuracy for the largest problems we test, which future QPU generations can potentially reach and outperform.

 \begin{table}[!t]
 \centering
 \caption{Average accuracy (1 is the best) for several methods on our Q-MSEG dataset. The highest accuracy is in boldface.  
    \label{tab_real}}
 \resizebox{0.95\columnwidth}{!}{
 \begin{tabular}{ l ccc c c c c c c c c c c c c c} 
 \hline\noalign{\smallskip}
  \# Qubits: & 96 & 102	&120&	126&	128&	136&	160&	168&	180&	190&	200&	216&	220&	243 \\
 \noalign{\smallskip}
 \hline
 \noalign{\smallskip} 
 Xu et al. \cite{XuCheongAl19} & 0.89 & 0.89 & 0.94 & 0.75 & 0.96 & 0.97 & 0.86 & 0.86 & 0.97 & 0.88 & 0.96 & 0.77 & 0.83 & 0.74 \\ 
\textsc{Mode} \cite{ArrigoniPajdla19a} & 0.93 &    0.93 &   0.96 & 0.93 &  \bf  0.97 &    0.97  & \bf 0.98 & 0.99 &  \bf  0.98 & \bf 0.99 &  \bf 0.99 &   0.93 &  \bf  1 &  \bf 0.94 \\
\textsc{Synch} \cite{ArrigoniPajdla19b} & 0.93 &    0.94 &    0.95 &    0.95 &    0.84 &    0.92 &    0.97 &  \bf 1 &    0.89 & 0.95 &    0.90 &   \bf  0.94 &    0.99 &    0.92 \\
 \ourmethod &  0.97	& 0.97	& 	0.97	& 	0.96	&	0.95	& \bf	0.98	& \bf	0.98	& 0.99	& \bf	0.98	& \bf	0.99	& \bf	0.99	&	0.64	&	--	& -- \\
\ourmethodsimple & 0.96	&  0.97	&	0.95	&	0.94	&	0.89	&	0.89	&	0.88	&	0.85	&	0.74	&	0.75	&	0.79	&	0.59	&	0.75	&	0.58 \\
\ourmethod, SA & 0.97	&  0.97	& 0.97	&  0.96	& 0.95	& \bf 0.98	&  \bf 0.98	& \bf 1	& \bf 0.98	& \bf 0.99	& \bf 0.99	& 0.68 & 0.98	& 0.72 \\
\ourmethodsimple, SA & \bf 0.98	& \bf 0.99	& \bf 0.99	& \bf 1	& 0.96	& \bf 0.98 & \bf 0.98	& \bf 1	&  0.94	& 0.97 	& \bf 0.99 	& 0.80	& \bf 1	& 0.59 \\ 
 \noalign{\smallskip} 
 \hline
 \end{tabular}
 }
  
 \end{table}

Results are given in Tab.~\ref{tab_real} which reports, for each configuration, the mean accuracy $\mu$ (over 20 problem instances) for all the analyzed methods. 
Results show that there is no clear winner, since none of the methods outperforms all others in all cases. 
In particular, \ourmethod is better than the state of the art on small-scale problems (\textit{i.e.,} 96-126 qubits) and comparable or better than the best traditional method (\textit{i.e.,} \textsc{Mode} \cite{ArrigoniPajdla19a}) on medium-scale problems (\textit{i.e.,} 136-200 qubits). On the largest cases, instead, either the performances of \ourmethod significantly drop (\textit{i.e.,} 216 qubits) or it was unable to find a minor embedding and, hence, 
a solution (\textit{i.e.,} 220-243 qubits). 
Concerning \ourmethodsimple, we can observe that---although working under easier assumptions---it is, in general, worse than \ourmethod. This might be caused by the difficulty of satisfying the additional (simplified)  constraints in practice (as they are treated as soft instead of hard). 
Note that the performances of our approaches are largely affected by the limitation of current QPU. In this respect, the on-par or higher performance of SA (see Tab.~\ref{tab_real}) suggests that our QUBO approach to motion segmentation is sound and that future QPU generations are expected to reach or surpass that result.
Qualitative results from an example with $96$ qubits are shown in  Fig.~\ref{fig:qualitative}: in this case, \ourmethod outperforms existing methods. 

Fig.~\ref{fig:chain_phys} visualises how the expected number of physical qubits and the maximum chain length are increasing with the increasing problem size.
It also reports the probability of finding a solution in a sampling, which is calculated as the portion of optimal (lowest-energy) solutions among all solutions. 
Note that \ourmethod has a non-zero probability to find an optimal solution for all problems except the largest embeddable one with $216$ qubits, whereas \ourmethodsimple has a non-zero probability only on the two smallest cases. 

 \begin{table}[!t]
 \centering
   \caption{Accuracy ($1.0$ is the best) for several methods on sub-problems sampled from the Hopkins dataset \cite{TronVidal07}. The highest accuracy is in boldface.  
   \label{tab_hopkins}}
 \resizebox{1\columnwidth}{!}{
 \begin{tabular}{ l ccc c c c c c c c c cccccccc c c c c} 
 \hline\noalign{\smallskip}
  \# Qubits: & 120   & 126   & 132  &  138   & 144  &  156   & 162  &  168   & 174  &  180  &  186   & 192  &  198   & 204   & 210  & 216  &  222  &  228  &  234   & 240\\
  \noalign{\smallskip}
 \hline
 \noalign{\smallskip} 
   Xu et al. \cite{XuCheongAl19} &     0.80 &    0.78  &   0.81  &   0.79  &   0.83   &  0.81 &    0.84  &   0.81  &   0.85  & 0.89  &   0.88  &   0.94 &    0.96  &   0.96  &   0.97 & \bf 1  &   0.98  & \bf  1 & \bf0.99  & \bf  1 \\
  \textsc{Mode} \cite{ArrigoniPajdla19a} &     0.89  &  0.91  &  0.90  &  0.93  &  0.92  &  0.94  &  0.95  &  0.95 &   0.96 & 0.95 &   0.97  &  0.98  &  0.98  &  0.98 &   0.98 &   0.99  &  0.98  &  0.98 &\bf 0.99  &  0.99 \\
  \textsc{Synch} \cite{ArrigoniPajdla19b}  &  0.87  &  0.93  &  0.95  &   0.96  & \bf 0.99  &  0.96  & \bf 0.99   & \bf0.99   & 0.96 & 0.99  & \bf 1  & \bf 1 &   0.99 & \bf  1 & \bf  1  &  0.70  &  0.97 &  \bf 1 & \bf 0.99  &  0.65 \\
   \ourmethod &  0.92 &  0.89 &  0.93  & 0.93  & 0.93  & 0.93 &  0.95  & 0.94 &   0.96 &  0.95  & 0.96 &  0.97 &  0.96 & - & -  & -  & -  & -  & -  & -  \\
\ourmethodsimple & 0.91 &  0.92  & 0.91  & 0.92 &  0.94 &  0.89  & 0.91 &  0.89 &  0.90  & 0.88 &  0.88 &  0.89 & 0.88 & 0.89  & 0.88  & -  & -  & -  & -  & - \\
\ourmethod, SA & {0.93} & {0.90} & {0.92} & {0.94}  & {0.93} & {0.94} &  {0.95}  & {0.96} & {0.96}  & {0.96}   & {0.98} & {0.98}  & {0.98}  & {0.98} & {0.99}  & {0.99}  & {0.98}  & {0.98}  & \bf{0.99}  & {0.99} \\
\ourmethodsimple, SA & \bf {0.96} & \bf {0.97}  & \bf {0.98}  &  \bf{0.98}  &  \bf {0.99}& \bf  {0.99} & \bf {0.99} &  {0.97} & \bf {0.99}  & \bf{1} & \bf{1} & \bf{1} &\bf {1} & \bf {1} & \bf{1} &\bf {1}  & \bf{0.99}  &\bf {1}  & \bf{0.99}  & \bf{1} \\
 \noalign{\smallskip} 
 \hline
 \end{tabular}
 }
 \end{table}

\textbf{Experiments on Hopkins Benchmark.}
Starting from the popular Hopkins155 dataset \cite{TronVidal07}, we created small problems (with 120-240 qubits) by sampling a subset of images/points from the \textit{cars2$\_$06$\_$g23} sequence (see our supplement for details). For each configuration, 20 instances were created, resulting in 400 examples in total. We used the fundamental matrix model to produce the pairwise segmentations which are given as input to \textsc{Mode} \cite{ArrigoniPajdla19a}, \textsc{Synch} \cite{ArrigoniPajdla19b} and our quantum methods. We also consider the method by Xu et al.~\cite{XuCheongAl19} and the simulated annealing (SA) solver \cite{Kirkpatrick1983,DWave_neal}, as done previously. Results are given in Tab.~\ref{tab_hopkins}, showing that we outperform Xu et al.~\cite{XuCheongAl19} (for number of qubits ${<}198$) and obtain comparable performance to classic methods \cite{ArrigoniPajdla19a,ArrigoniPajdla19b} in several cases. 
Results for SA show that the global optima of our QUBOs match expected solutions (and that the obtained energy landscapes are as expected), even though the solutions cannot be recovered by QA in all cases due to hardware limitations. In particular, for small problems, the simulated annealing performs on par with QA (as expected), whereas large problems can be solved with the SA only.

\textbf{Parameters.}
In the tests with Q-MSEG and Hopkins, we use $\lambda_1 = 10.0$ (\textsc{v1}) and $\lambda_2 = 10.0$, $\lambda_3 = 4.0$ (\textsc{v2}). 
For the synthetic experiments, we set the parameters as follows: $\lambda_1 = 14.0$ (\textsc{v1}) and $\lambda_2 = 27.5$, $\lambda_3 = 3.2$  (\textsc{v2}). 
The chain strength required to keep the chains of physical qubits intact, is chosen as a linear function $c(k) = a k + b$ in the number of the logical qubits. 
We set $a = 0.13508$, $b =  -1.94207$ and $a = 0.1238$, $b = -1.3180$ for \textsc{v1} and \textsc{v2}, respectively. 

\textbf{Discussion.}
In our experiments, \textsc{QuMoSeg} achieves on-par performance\footnote{Note that this reflects the current situation in the field: indeed, other quantum methods \cite{BirdalGolyanikAl21,SeelbachBenkner2021} do not outperform classical methods in all scenarios too.} with non-quantum state of the art, showing the viability of a quantum approach for segmentation. 
Among the two proposed variants, \ourmethod is the most accurate and  \ourmethodsimple can solve larger problems. 
Moreover, our approach in a combination with SA shows highly promising results on both Q-MSEG and the data we generated from Hopkins.
The characteristics of the current quantum hardware starkly  influence the performance of \textsc{QuMoSeg}: Problems with ${>}120$ points cannot be minor-embedded, and the largest mappable problems require maintaining too long qubit chains, which impedes the optimum search. As many other quantum algorithms, \textsc{QuMoSeg} will benefit from improved quantum hardware, both in terms of accuracy and solvable problem sizes. 
{
Indeed, it is expected that the possibility to solve combinatorial problems without approximation will give practical advantages for large-scale problems in the future.}

\section{Conclusion}\label{sec:conclusion} 

We propose the first motion segmentation approach for an adiabatic quantum computer, 
which shows highly promising results and reaches state-of-the-art accuracy on a wide range of problems. 
We hope that the demonstrated progress encourages more work on quantum computer vision in the future. 

\bibliographystyle{splncs04}
\bibliography{Definitions,NostriNEW,AndreaNEW,LucaNEW,FedeNEW,quantum}

\clearpage

\appendix
\section{Supplementary Material}

This appendix provides further details on our experiments.

\smallskip
\textbf{Competing methods.} \textsc{Mode} \cite{ArrigoniPajdla19a}, \textsc{Synch} \cite{ArrigoniPajdla19b} and Xu et al. \cite{XuCheongAl19} are evaluated in Matlab on a 2020 MacBook Pro with 1.4 GHz processor and 8 GB RAM.

\smallskip
\textbf{Q-MSEG Dataset.} 
We generate a new dataset for motion segmentation with ground-truth annotations, which comprises six images depicting three planar objects captured from diverse viewpoints, shown in Fig.~\ref{fig:dataset}. 
Each object is manually annotated with respectively 10, 11 and 12 keypoints, selected on highly-textured locations. 
Several motion segmentation problems are derived from the dataset by randomly sampling a subset of points/images/motions. In particular, we consider 14 different choices for the number of points/images/motions (see Table \ref{tab_dataset}), and we sample 20 problem instances for each configuration, resulting in 280 problem instances in total. See Tab.~\ref{tab_real_supp} and the main paper for results.

\begin{table}[!h]
\caption{
Statistics of the Q-MSEG Dataset. Each configuration has $n$ images and $d$ motions. In each image, there are $m_1$ points in the first motion, $m_2$ in the second motion and $m_3$ in the third motion (if available). The total number of logical qubits is reported for each configuration, which is given by $ dn(m_1+m_2+m_3)$. 
\label{tab_dataset}}
\centering
\resizebox{0.7\linewidth}{!}{
\begin{tabular}{ c ccc c c c c c c c c c c c c c} 
\hline\noalign{\smallskip}
  \# Qubits: & 96 & 102	&120&	126&	128&	136&	160&	168&	180&	190&	200&	216&	220&	243 \\ 
\noalign{\smallskip}
\hline
\noalign{\smallskip} 
 $n$   &3& 3  & 3  & 3  & 4  & 4 & 4  & 4  & 5 & 5  & 5 & 3 & 5 & 3 \\ 
 $d$   &2& 2  & 2  & 2  & 2  & 2 & 2  & 2  & 2 & 2  & 2 & 3 & 2 & 3 \\ 
 $m_1$ &8 &  8& 10 & 10 & 8  & 8 & 10 & 10 & 9 & 9  & 10& 8 & 11 & 8 \\
 $m_2$ &8  &9 & 10 & 11 & 8  & 9 & 10 & 11 & 9 & 10 & 10& 8 & 11 & 9 \\
 $m_3$ &--&-- & -- & -- & -- & --& -- & -- & --& -- & --& 8 & -- & 10 \\
\noalign{\smallskip}
\hline
\end{tabular}
}
\end{table}

\begin{figure}[h!] 
  \centering 
\includegraphics[width=0.32\linewidth]{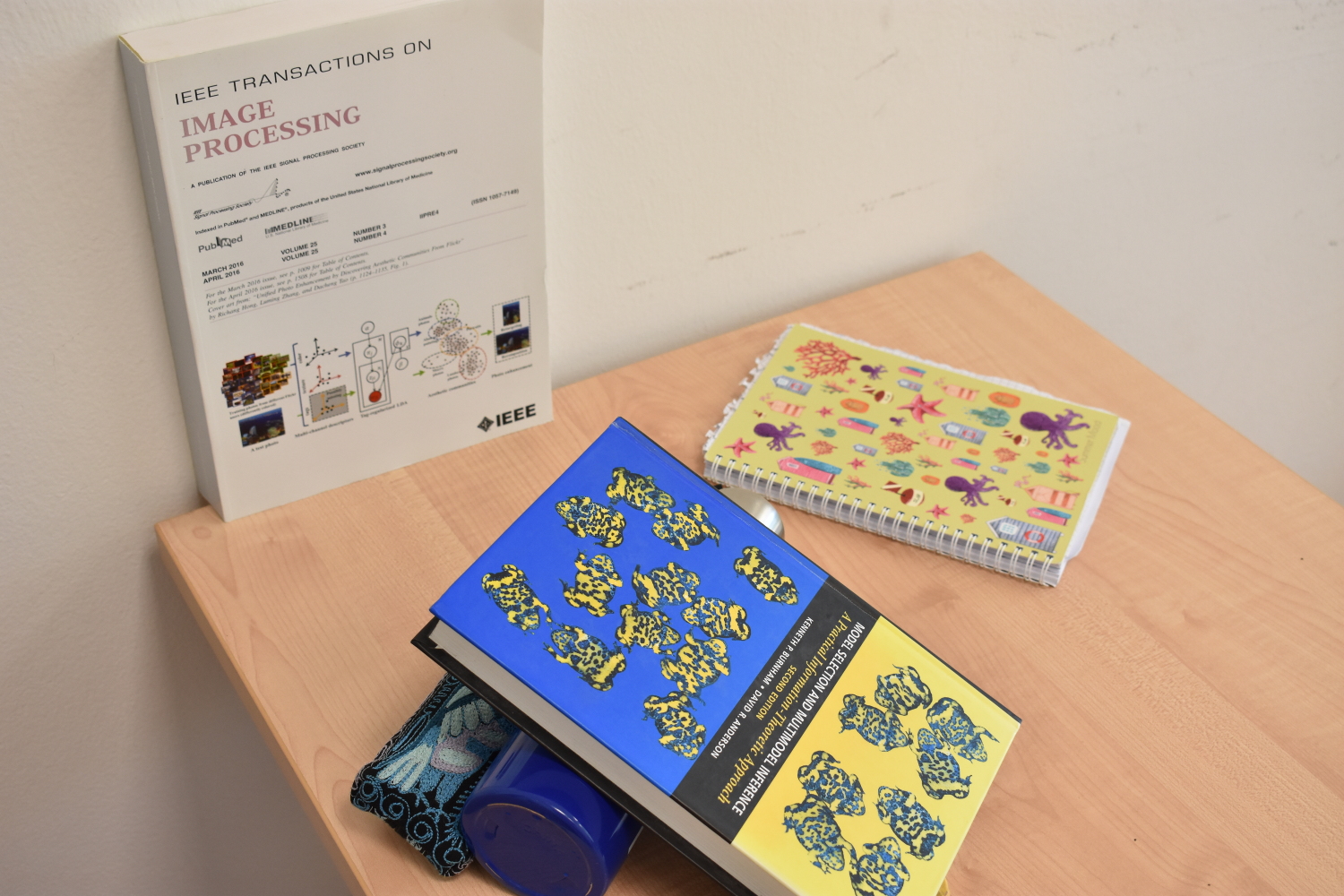} 
\includegraphics[width=0.32\linewidth]{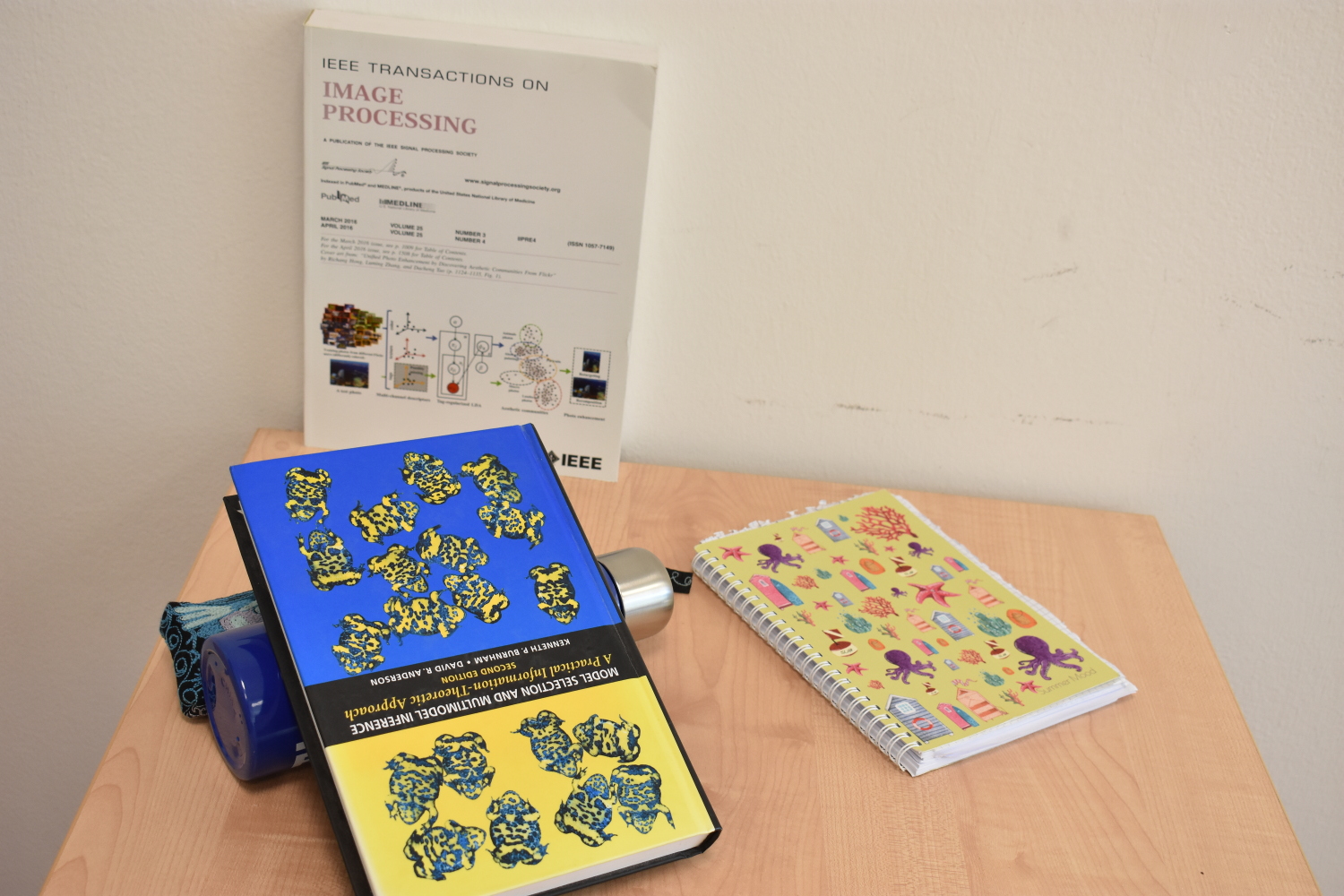} 
\includegraphics[width=0.32\linewidth]{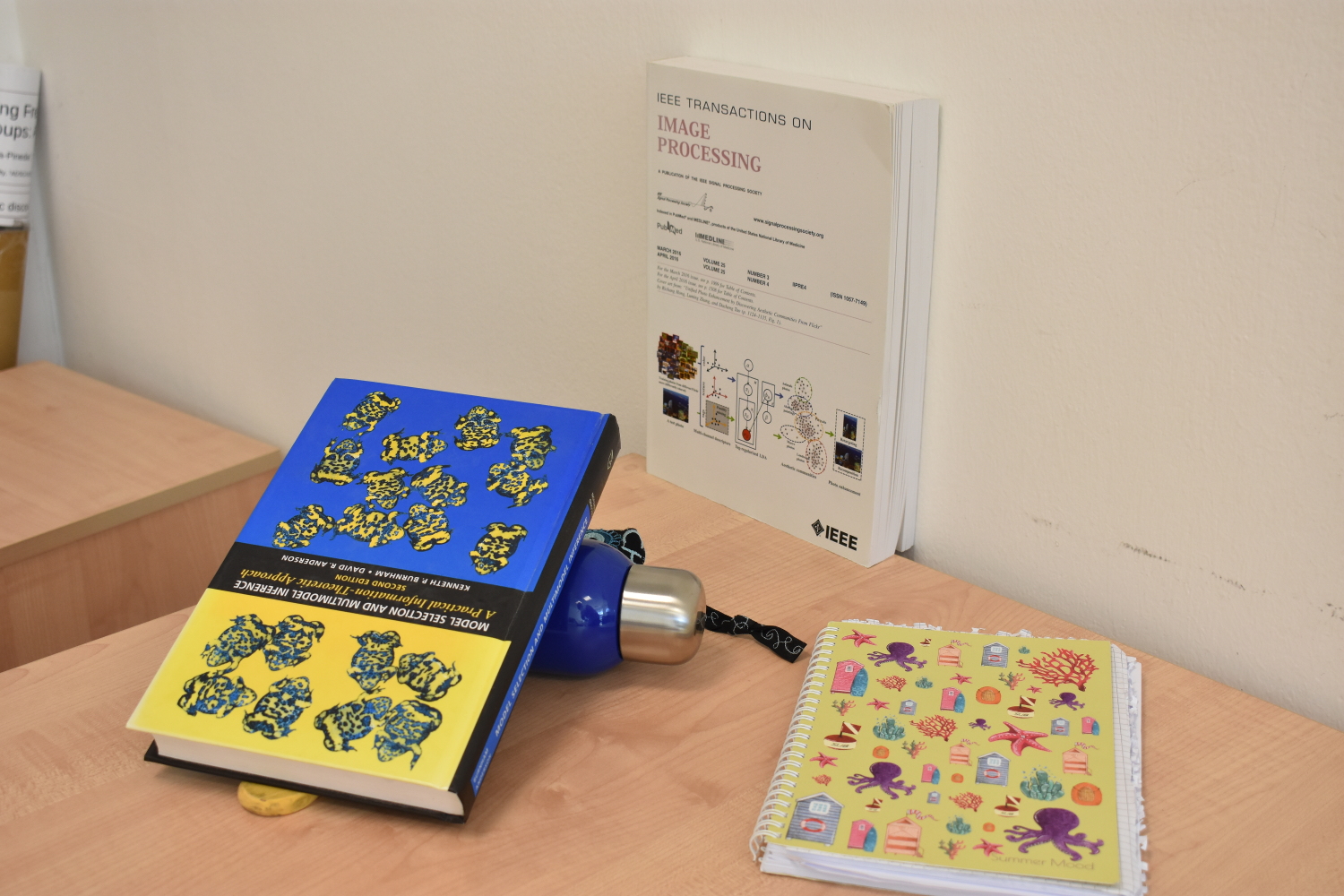}
\includegraphics[width=0.32\linewidth]{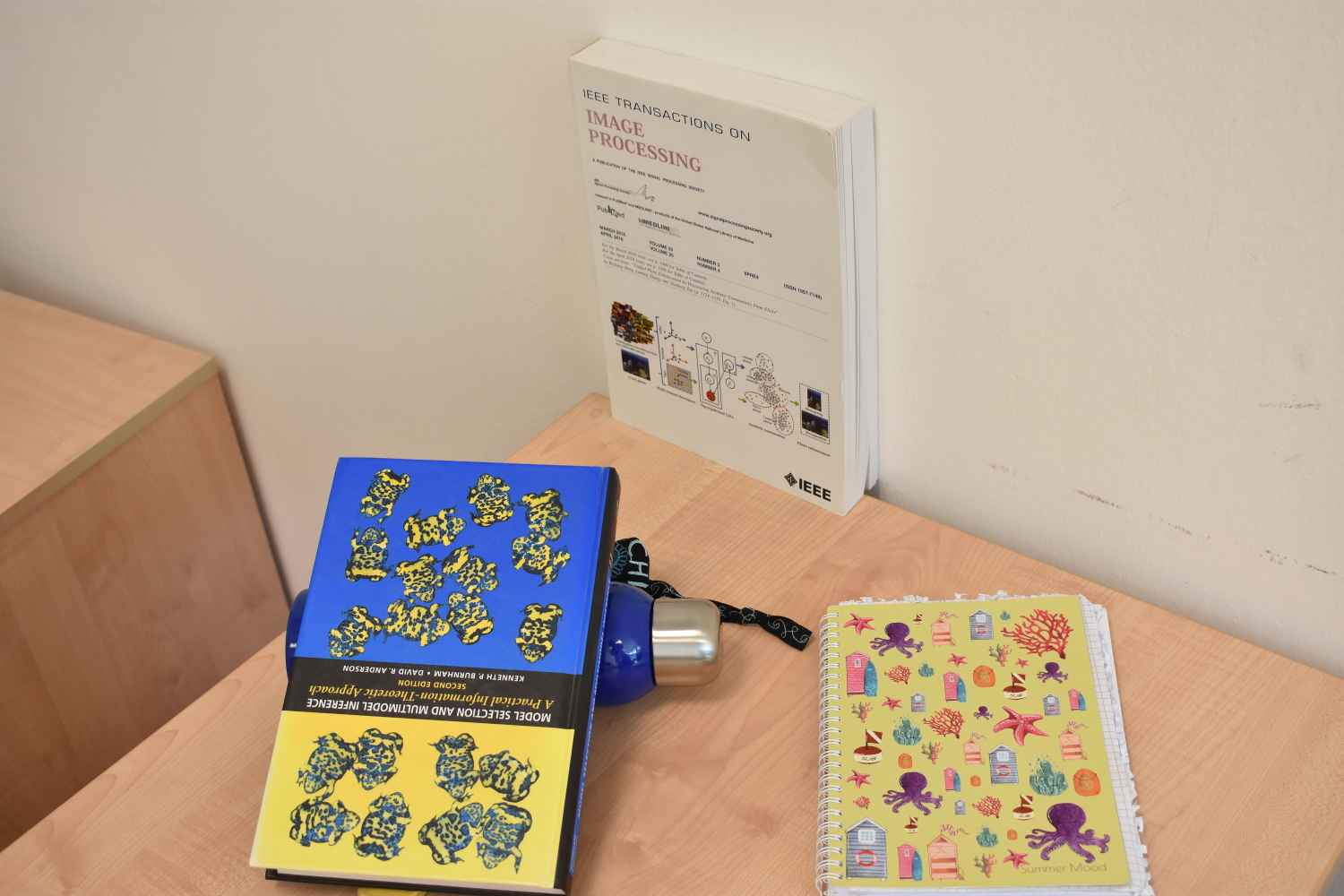} 
\includegraphics[width=0.32\linewidth]{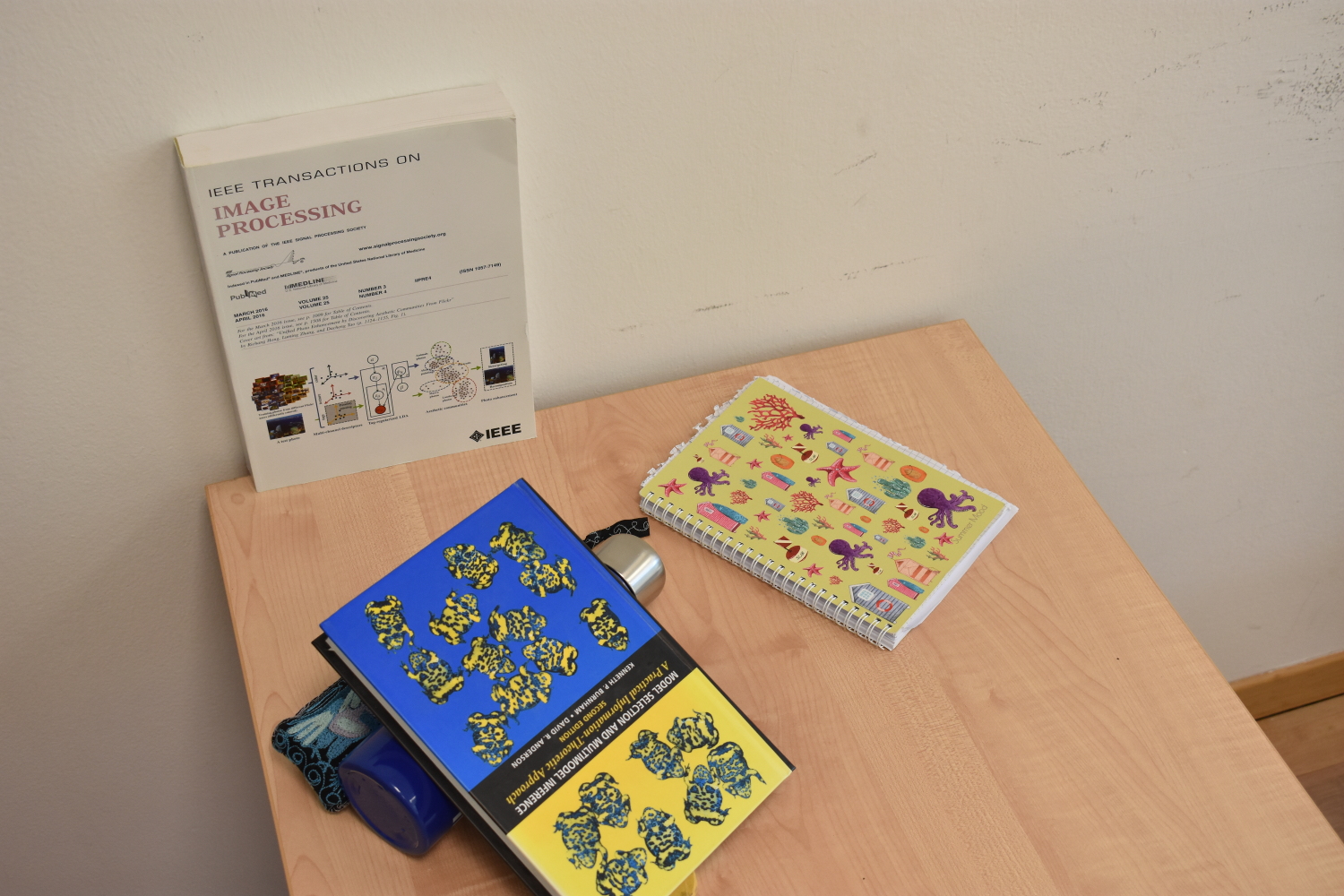} 
\includegraphics[width=0.32\linewidth]{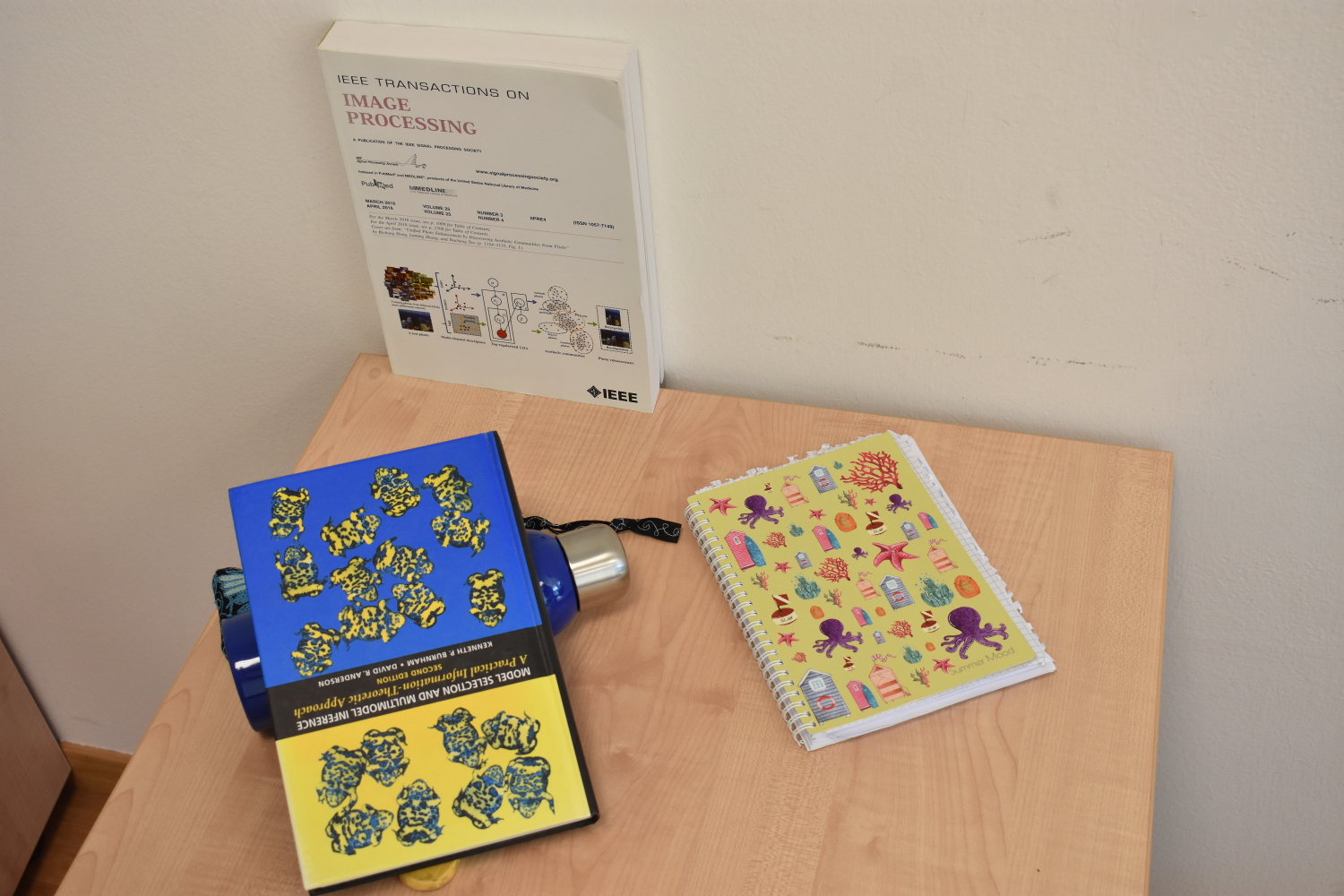} 
\caption{Images from the Q-MSEG dataset depicting three planar motions.} 
\label{fig:dataset} 
\end{figure}

\smallskip
\textbf{Parameter Selection.} To select parameters for our method, we perform a systematic (exhaustive) grid search on small problem instances (with 24 bits). 
Talking in terms of the energy landscape, the purpose of this step is to find ranges of $\lambda$ which push the lowest-energy solution of the target QUBO as close as possible to the desired ground-truth solution. 
We observe that the obtained parameter ranges generalise well to larger problems when solved on a QPU (\textit{i.e.,} all tested problem sizes ranging from $96$ qubits to $243$ qubits). 
A similar observation was made in QSync \cite{BirdalGolyanikAl21}. 
Note that parameters can be selected in comparably large ranges, which allows safely using the same parameters across many---and much larger---problems. 
In particular, the parameters selected for our Q-MSEG dataset generalize well to Hopkins data without any tuning. 
Concerning the chain strength, the coefficients $a$ and $b$ of function $c$ (see the main paper) are calculated by a linear regression applied to the measured average maximum chain lengths observed in exemplary problems of different sizes (see Fig.~\ref{fig:ch_str}).

\begin{figure}[t!] 
  \centering 
\includegraphics[width=0.33\linewidth]{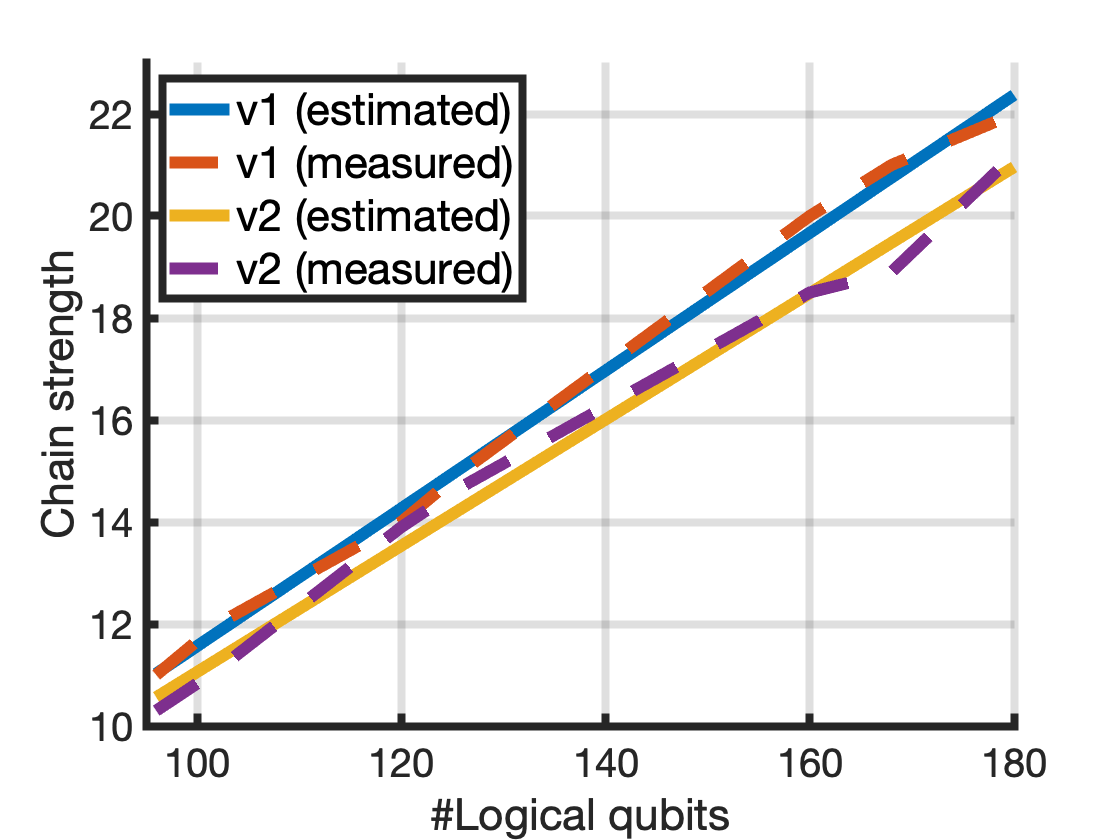}
\caption{
Chain strength versus number of logical qubits in Q-MSEG dataset. 
} 
\label{fig:ch_str} 
\end{figure}

 \begin{table*}
   \caption{
The accuracy ($1.0$ is the best) and its st.~dev. for several methods on Q-MSEG dataset. 
For \textsc{QuMoSeg}, we also report the number of found ground-truth solutions (``$\#$sol'') and the probability (``prob'') of measuring them as lowest-energy samples, and we perform evaluation both on Adv1.1 and Adv4.1. 
  \label{tab_real_supp}}
 \centering
 \resizebox{0.99\linewidth}{!}{
 \begin{tabular}{ l l ccc c c c c c c c c c c c c c} 
 \hline\noalign{\smallskip}
  \# Qubits: &  & 96 & 102	&120&	126&	128&	136&	160&	168&	180&	190&	200&	216&	220&	243 \\
 \noalign{\smallskip}
 \hline
 \noalign{\smallskip} 
 \multirow{2}{*}\textsc{Mode} \cite{ArrigoniPajdla19a} & acc & 0.93 &    0.93 &   0.96 & 0.93 &  \bf  0.97 &    0.97  & \bf 0.98 &    \bf 0.99 &  \bf  0.98 & \bf 0.99 &  \bf 0.99 &    \bf0.93 &  \bf  1 &  \bf 0.94 \\
 &  std & 0.07 &    0.06 &    0.05 &    0.06 &    0.04 &    0.05 &    0.02 &    0.01 &    0.02 & 0.007  &  0.01  &  0.04 &         0 &   0.04 \\
 \noalign{\smallskip}
 \hline
 \noalign{\smallskip} 
 \multirow{2}{*} \textsc{Synch} \cite{ArrigoniPajdla19b} & acc & 0.93 &    0.94 &    0.95 &    0.95 &    0.84 &    0.92 &    0.97 &    1 &    0.89 & 0.95 &    0.90 &    0.94 &    0.99 &    0.92 \\
  &  std & 0.14 &    0.15 &    0.08 &    0.14 &    0.21 &    0.15 &    0.03 &         0 &   0.15 & 0.15 &    0.20 &    0.06 &    0.003 &    0.11\\
   \noalign{\smallskip}
 \hline
 \noalign{\smallskip} 
  \multirow{2}{*}   {Xu et al. \cite{XuCheongAl19}} & acc & 0.89 & 0.89 & 0.94 & 0.75 & 0.96 & 0.97 & 0.86 & 0.86 & 0.97 & 0.88 & 0.96 & 0.77 & 0.83 & 0.74 \\ 
 & std  &  0.18   &   0.12   &   0.08  &    0.14   &   0.05   &   0.04   &   0.18   &   0.16  &    0.04  & 0.10   &   0.06  &    0.09   &   0.15   &   0.05 \\
       \noalign{\smallskip}
 \hline
 \noalign{\smallskip} 
 & acc & \bf 0.97	&	\bf 0.97	& \bf	0.97	& \bf	0.96	&	0.95	& \bf	0.98	& \bf	0.98	& \bf	0.99	& \bf	0.98	& \bf	0.99	& \bf	0.99	&	0.64	&	--	& -- \\
\ourmethod &  std & 0.04	&	0.03	&	0.03	&	0.05	&	0.11	&	0.03	&	0.02	&	0.02	&	0.02	&	0.01	&	0.01	&	0.10	&	--	&	-- \\
( Adv4.1 ) &  $\#$sol & 9	&	5	&	6	&	11	&	6	&	11	&	8	&	9	&	7	&	6	&	9	&	0	&	--	& --	\\
  &  prob & 0.16	&	0.07	&	0.03	&	0.03	&	0.04	&	0.05	&	0.008	&	0.0095	&	0.0097	&	0.0007	&	0.006	&	0	&	--	&	-- \\
   \noalign{\smallskip}
 \hline
 \noalign{\smallskip} 
& acc & 0.96	&	\bf 0.97	&	0.95	&	0.94	&	0.89	&	0.89	&	0.88	&	0.85	&	0.74	&	0.75	&	0.79	&	0.59	&	0.75	&	0.58 \\
\ourmethodsimple  &  std &  0.03	&	0.02	&	0.03	&	0.02	&	0.09	&	0.04	&	0.03	&	0.10	&	0.12	&	0.07	&	0.04	&	0.10	&	0.06	&	0.07\\
 ( Adv4.1 ) &  $\#$sol & 2 & 2 & 0& 0& 0& 0& 0& 0 & 0& 0& 0& 0& 0& 0 \\
  &  prob & 0.0001 & 0.0001 & 0 & 0& 0& 0& 0& 0& 0 & 0& 0& 0& 0& 0\\
     \noalign{\smallskip}
 \hline
 \noalign{\smallskip} 
 & acc & \bf0.97 & \bf	0.97 & \bf	0.97 &\bf 	0.96 & 	0.94	 & 0.97	 & 0.97 & 	0.97 & 	0.97 	& -- &	--	& --&	--	& -- \\
\ourmethod &  std & 0.04 & 	0.02 & 	0.03 & 	0.05 & 	0.11 & 	0.04 & 	0.03 & 	0.03	 & 0.03 &	--	&	-- & -- &	--	& --\\
( Adv1.1 ) &  $\#$sol & 9	&	5	&	6	&	11	&	6	
&	9	&	5	&	3	&	5	&	--	&	--	&	--	&	--	& --	\\
  &  prob & 0.13	&	0.05	&	0.03	&	0.01	&	0.03	&	0.02	&	0.0003	&	0.0002 &	0.0007	&	--	&	-- &	--	&	-- & -- \\
     \noalign{\smallskip}
 \hline
 \noalign{\smallskip} 
& acc &  0.95 &	0.96	 &	0.93 &		0.92 &		0.86 &		0.86 &		0.97 &		0.81 &		0.69 &		0.67	 &	0.71	 &	0.59	 &	0.66 & --	 \\
\ourmethodsimple  &  std & 0.05 &		0.02	 &	0.03	 &	0.03 &		0.09 &		0.06 &		0.03 &		0.07 &		0.07 &		0.08	 &	0.05 &		0.08 &		0.05 & --	\\
 ( Adv1.1 ) &  $\#$sol & 1 & 0 & 0& 0& 0& 0& 5& 0 & 0& 0& 0& 0& 0& -- \\
  &  prob & 0.0001 & 0 & 0 & 0& 0& 0& 0.0003 & 0 &0 & 0& 0& 0& 0& --\\
 \noalign{\smallskip}
 \hline
 \end{tabular}
 }
 \end{table*}

\smallskip

\textbf{Evaluation on Adv1.1.} 
Table \ref{tab_real_supp} reports our evaluation on D-Wave Advantage1.1 (Adv1.1). 
For reference, we also report results on Adv4.1 and those  of competing methods, which are copied from the main paper. 
Results show that the accuracy on Adv1.1 is slightly lower than on Adv4.1, and the latter can solve larger problems. 
This is also reflected in the number of found ground-truth solutions and the probability of finding a ground-truth solution. 
Moreover, larger problems can be minor-embedded to Adv4.1  compared to Adv1.1. 
We also observe that for small problems, the accuracy of our method on Adv1.1 and Adv4.1 is similar, whereas for larger problems, Adv4.1 performs consistently better. 
This observation agrees with the statement done in 
Appendix \textit{A.1 Performance on Native Inputs} of 
the technical report by McGeoch and Farr\'e \cite{PerformanceUpdate}.

\smallskip
\textbf{Hopkins Dataset.} 
Starting from the well-known Hopkins155 dataset \cite{TronVidal07}, we create small problems (with 120-240 qubits) by sampling a subset of images/points from the \textit{cars2$\_$06$\_$g23} sequence, which represents two moving objects in an outdoor environment. See Tab.~\ref{tab_dataset_hopkins} for the details. For each configuration, 20 instances were created, resulting in 400 examples in total. See the main paper for results on this dataset.

\begin{table}[htbp]
\caption{
Statistics of sub-problems sampled from the Hopkins dataset \cite{TronVidal07}. Each configuration has 3 images and 2 motions. In each image, there are $m_1$ points in the first motion and $m_2$ in the second motion. The total number of logical qubits is reported for each configuration, which is given by $ 6(m_1+m_2)$. 
\label{tab_dataset_hopkins}}
\centering
\resizebox{0.9\linewidth}{!}{
\begin{tabular}{ c ccc c c c c c cccc c c c c c cccc c} 
\hline\noalign{\smallskip}
  \# Qubits: & 120   & 126   & 132  &  138   & 144  &  156   & 162  &  168   & 174  &  180  &  186   & 192  &  198   & 204   & 210  & 216  &  222  &  228  &  234   & 240\\
\noalign{\smallskip}
\hline
\noalign{\smallskip} 
 $m_1$ &10 & 10  & 11 &  11 &  12 &  13 &  13 &  14 &  14 &  15 &  15  & 16  & 16  & 17 &  17  & 18 &  18 &  19  & 19 &  20 \\
 $m_2$ &10& 11 &11  & 12  & 12  & 13 &  14 &  14 &  15 &  15  & 16 &  16 &  17 &  17 &  18  & 18  & 19  & 19  & 20 &  20\\
\noalign{\smallskip}
\hline
\end{tabular}
}
\vspace{10pt}
\end{table}

\smallskip 
\textbf{Minor Embeddings.} 
Fig.~\ref{fig:minor_embedding} shows an exemplary minor embedding for a problem with $96$ qubits. Fig.~\ref{fig:minor_embedding}-(left) is a photograph of the Adv4.1 processor. Fig.~\ref{fig:minor_embedding}-(middle) and \hbox{-(right)} show the qubits graph of the logical problem ($96$ nodes) and its embedding into the processor ($803$ nodes), respectively. 
Each node in the logical problem graph is colored according to whether the qubit was measured as one (yellow) or zero (white). The figure on the right shows a section of the processor centered around the region with active qubits. The colors denote qubits measured---according to the Ising model---as ${+}1$ (blue) or ${-}1$ (white).

\begin{figure}[t!] 
  \centering 
\includegraphics[width=1.0\linewidth]{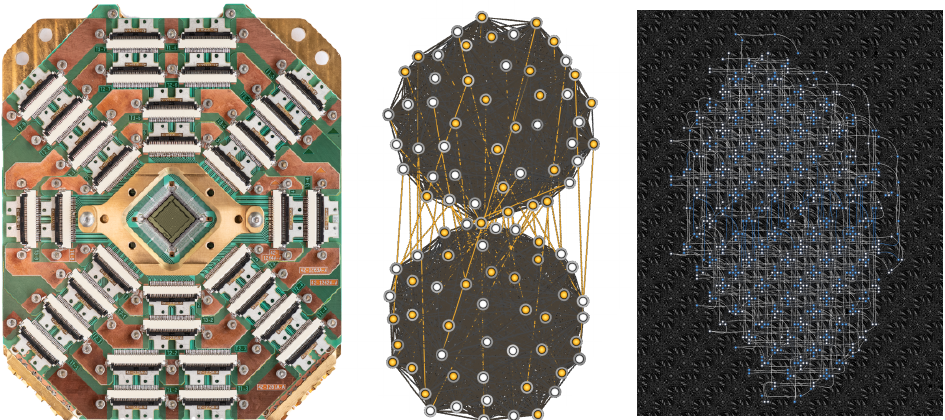} 
\caption{D-Wave Adv4.1 quantum annealer (left; reproduced with permission of D-Wave Systems), graph of $96$ logical problem qubits (middle) and its minor embedding with $803$ physical nodes (right).} 
\label{fig:minor_embedding} 
\end{figure}

\end{document}

%% file: related_work.tex
\section{Related Work}
\label{sec:related_work} 
After having introduced \textsc{QuMoSeg}, we next  review the most related methods.

\textbf{Synchronization} refers to recovering elements of a group (associated to vertices in a graph) starting from a (redundant) set of pairwise ratios (associated to edges in the graph). 
Popular synchronization problems involve \emph{rotations} and \emph{rigid motions} \cite{Singer11,BernardThunbergAl15,ChatterjeeGovindu13,TronDanilidis14,BirdalArbelAl20,ErikssonOlssonAl18,RosenCarloneAl19,ThunbergBernardAl17,LiLingAl21}, which are at the core of tasks such as camera motion estimation in SfM \cite{OzyesilVoroninskiAl17}, point cloud registration \cite{GovinduPooja14} and SLAM \cite{CarloneTronAl15}. 
Other synchronization problems concern \emph{homographies} (which are related to image mosaicking \cite{SchroederBartoliAl11}) and \emph{affine transformations} (which were used to solve for global color correction \cite{SantellaniMasetAl18}).
Although synchronization has a well-established theory for the case where unknowns/measures belong to a group, specific routines can be possibly developed when the variables do not belong to a group and have a poorer structure. 
A notable example are \emph{partial permutations}, which appear in the context of multi-view matching \cite{PachauriKondorAl13,ChenGuibasAl14,ZhouZhuAl15,MasetArrigoniAl17,BernardThunbergAl19}. 
Another example are \emph{binary matrices}, which are related to part segmentation in point clouds \cite{HuangWangAl21} and motion segmentation in images \cite{ArrigoniPajdla19b} (the basis of our method).

\textbf{Motion Segmentation} aims at detecting moving objects in a scene given multiple images by grouping all the key-points moving in the same way \cite{SaputraMarkhamAl18}. 
Existing techniques fall into three main categories, depending on the assumptions made on the inputs.
Some methods assume extracted key-points and work with \emph{unknown} correspondences \cite{JiLiAl14,WangLiuAl18}. 
Other algorithms assume some \emph{local} information about correspondences in addition to the key-points, namely matches between image pairs \cite{ArrigoniPajdla19a,ArrigoniPajdla19b}.
The third category assumes known \emph{global} information about the correspondences, namely multi-frame {trajectories} \cite{YanPollefeys06,RaoTronAl10,LiGuoAl13,JiSalzmanLi15,LaiWangAl17,XuCheongAl19}. 
Our method belongs to the second category \cite{ArrigoniPajdla19a,ArrigoniPajdla19b}. 
Accordingly, our technique splits motion segmentation into multiple two-frame sub-problems and then finds a global consistency among the partial results. 
Our method is related to Arrigoni and Pajdla  \cite{ArrigoniPajdla19b}, as we adapt their model and the data structures. 
Their approach relies on a similar formulation as \eqref{cost_trace_final} but solves the problem over real variables instead of binary ones, ending up with an approximate solution based on a spectral decomposition. 
In this paper, instead, we derive a QUBO formulation from scratch and solve it \emph{without relaxation} on an AQC.
%

\textbf{Quantum Computer Vision.} Several quantum techniques are available for computer vision tasks, such as recognition and classification \cite{OMalley2018,Cavallaro2020}, object tracking \cite{LiGhosh2020}, 
transformation estimation \cite{golyanik2020quantum}, point set and shape alignment \cite{NoormandipourWang2021,SeelbachBenkner2021}, graph matching \cite{seelbach20quantum} and permutation synchronization \cite{BirdalGolyanikAl21}. 
Most of these methods are designed for an AQC. 
In \cite{OMalley2018}, a binary matrix factorization is applied to feature extraction from facial images, while in \cite{LiGhosh2020} redundant detections are removed in multi-object detection with the help of an AQC.
Another method classifies multi-spectral images  with quantum SVM \cite{Cavallaro2020}. 
A quantum approach for correspondence problems on point sets \cite{golyanik2020quantum} recovers rotations between pairs of point sets, which are approximated as sums of basis matrix elements. 
The qKC method described in \cite{NoormandipourWang2021} employs both classical and quantum kernel-based losses for point set matching. 
In contrast to \cite{golyanik2020quantum}, qKC is designed for a circuit-based quantum computer.  
QGM is the first approach for matching small graphs using AQC   \cite{seelbach20quantum}. 
Q-Match \cite{SeelbachBenkner2021} can non-rigidly match 3D shapes with up to $500$ points and it overcomes the hardware limitations of the modern AQC (\textit{e.g.,} the qubit connectivity pattern) by an iterative optimization scheme.

The quantum method most closely related to ours is QSync \cite{BirdalGolyanikAl21}, as it also uses the framework of synchronization but solves a \emph{different} problem. 
There are substantial differences between 
our task and permutation synchronization \cite{BirdalGolyanikAl21}: 
QSync operates on permutation matrices (\textit{i.e.,} both rows and columns sum to 1), whereas the binary matrices in our case are less constrained (only rows sum up to 1). 
Although this might seem a minor difference, the term $\trace(X_jX_i^\tr X_i X_j^\tr)$ in QSync is constant and can be ignored in \eqref{eq_synch_trace_NEW}; the absence of ``columns sum to 1'' constraint means that we either i) require additional assumptions (\ourmethodsimple) or ii) resort to different computations resulting in a new (dense) matrix (\ourmethod). 
In QSync, the number of variables scales quadratically with the number of points, whereas our method is more efficient per construction (\textit{e.g.,} it uses matrices with $2k$ entries if there are two motions, see Remark \ref{remark_nopoints}). Hence, we can handle larger problems: five images of two motions with ten points each, \textit{i.e.,} $100$ points (see Sec.~\ref{sec:experiments}) \textit{vs} five points in four images \cite{BirdalGolyanikAl21}. 
%